\documentclass[letterpaper,10pt]{article}
\usepackage{style}
\newgeometry{vmargin={40mm}, hmargin={25mm,25mm}}
\usepackage[affil-it]{authblk}
\usepackage[numbers]{natbib}

\title{Multi-armed Bandits with Missing Outcome}
\makeatletter
\renewcommand{\@fnsymbol}[1]{%
   \ifcase#1\or $\dagger$\or *\fi}
\makeatother
\author[1]{Ilia Mahrooghi\thanks{This work was completed during the authors' tenure as research interns at EPFL.}}
\author[1]{Mahshad Moradi$^\dagger$}
\author[2]{Sina Akbari}
\author[2]{Negar Kiyavash}
\affil[2]{EPFL, Switzerland}
\affil[1]{\small\texttt{\{ilia.mahrooghi2003,mahshadmoradix\}@gmail.com}}
\affil[2]{\texttt{\{sina.akbari,negar.kiyavash\}@epfl.ch}}
\date{}
\begin{document}
\maketitle

\begin{abstract}
    While significant progress has been made in designing algorithms that minimize regret in online decision-making, real-world scenarios often introduce additional complexities, perhaps the most challenging of which is missing outcomes.
    Overlooking this aspect or simply assuming random missingness invariably leads to biased estimates of the rewards and may result in linear regret.
    Despite the practical relevance of this challenge, no rigorous methodology currently exists for systematically handling missingness, especially when the missingness mechanism is not random.
    In this paper, we address this gap in the context of multi-armed bandits (MAB) with missing outcomes by analyzing the impact of different missingness mechanisms on achievable regret bounds.
    We introduce algorithms that account for missingness under both missing at random (MAR) and missing not at random (MNAR) models.
    Through both analytical and simulation studies, we demonstrate the drastic improvements in decision-making by accounting for missingness in these settings.
\end{abstract}
\section{Introduction}
Multi-armed bandit (MAB) algorithms have emerged as indispensable tools for decision-making under uncertainty, balancing the trade-off between exploring different options and exploiting the best known action.
These algorithms have achieved success in various domains ranging from personalized online advertisement and recommender systems \cite{li2010contextual, xu2020contextual, ban2024neural} to clinical trials \cite{villar2015multi, aziz2021multi, varatharajah2022contextual} and adaptive routing in communication systems \cite{maghsudi2016multi,li2020multi}.
For instance, in online advertising, advertisers need to continuously select which ad to show to a user to maximize click-through rates.
Similarly, in clinical trials, researchers must decide which treatment to administer to patients to optimize recovery rates.
MAB algorithms guide decision-makers in such scenarios to learn actions that minimize regret.

Significant progress has been made in developing MAB algorithms that minimize regret in various settings \cite{lai1985asymptotically,Auer2002,bubeck2012regret,lattimore2020bandit,slivkins2019}.
However, the real world often introduces challenges that deviate from the assumptions of the classical MAB framework or its current extensions.
One of the most critical challenges is that of missing outcomes -- situations where the results of certain actions are not always observed.
This challenge arises more often than not in practice and can fundamentally undermine the decision-making process if left unaddressed. To illustrate this, consider an example of a large-scale clinical trial for a new cancer treatment.
Patients are randomly assigned to different treatment arms, and their health outcomes are monitored over time. In practice, not all patients will complete the trial. Some may drop out early due to side effects, while others may stop reporting outcomes for personal reasons, and some could pass away during the trial due to reasons not related to the treatment (competing events).
Crucially, the missingness of the outcome may not be random.
Patients experiencing severe side effects or poor health are more likely to drop out, meaning that the missingness mechanism is correlated with the unobserved outcome itself.
This introduces systematic bias into the estimation of the rewards, and if not accounted for, would lead to poor decision-making.

The issue of missingness is not confined to healthcare.
In a recommendation system that suggests articles to users on an online platform,
if users who find the content irrelevant are less likely to provide feedback (e.g., they leave the site without interacting), the system could overestimate the value of the recommended articles, assuming that the missing feedback is independent of user satisfaction.
Here too, missingness is correlated with the unobserved outcome, leading to biased reward estimates and sub-optimal recommendations.

The problem of missing data is a fundamental challenge in causal inference.
This issue has been extensively studied over the past decades, with seminal works such as  \cite{rubin1976inference, little2019statistical, bang2005doubly} laying the foundation for dealing with biased estimations in the presence of missing data. 
These methods, along with more recent developments in graphical models for handling missing data \cite{mohan2021graphical, nabi2020full}, have become standard approaches in causal inference.
Missing data has also been extensively explored in specific contexts such as instrumental variables, \cite{tchetgen2017general, sun2018semiparametric, kennedy2019handling}
and mediation analysis \cite{zhang2013methods, zhang2015mediation, kidd2023mediation}, among others.
By contrast, the challenge of missing outcomes has received relatively little attention in multi-armed bandit problems, although some progress has been made in related areas.
For instance, the problem of delayed feedback in bandits bears some similarity to our setting, as both involve incomplete information at decision time.
Several works have addressed stochastic bandits with unrestricted delays \cite{joulani2013online, vernade2017stochastic}, and delays dependent on stochastic rewards \cite{pike2018bandits, lancewicki2021stochastic}.
In contextual bandits, \cite{bouneffouf2017context} studied linear contextual bandits with missing (restricted) contexts. 
While this work addresses missing data in bandits, it focuses on missing contexts rather than outcomes and assumes a linear reward model.
Others have explored bandit problems with variable costs or restricted observations.
For example, \cite{ding2013multi, seldin2014prediction} studied MAB problems with variable costs, where the outcome is observable only after paying the associated cost.


The closest to our paper are the works
\cite{chen2022some} and \cite{bouneffouf2020contextual} which consider the problem of MAB with missing outcomes. 
\cite{chen2022some} settles for some empirical considerations and suggestions, without formally studying the problem or providing tailored algorithms.
\cite{bouneffouf2020contextual} utilizes ideas from unsupervised learning to impute the missing rewards in a contextual bandit setting.
Both of these work under the assumption that the missingness mechanism is random, possibly after conditioning on the context.
In this paper, we do a thorough study from a formal perspective, characterizing the best achievable regret bounds under multiple scenarios with missing outcomes at random (MAR) as well as not at random (MNAR) models.
We present novel regret lower bounds and provide algorithms that are guaranteed to achieve optimal regret.
Another closely related line of work in the MAB research involves leveraging auxiliary information to enhance decision-making. 
Recent studies have explored the use of correlated auxiliary feedback to improve reward estimates \cite{verma2024exploiting}, introduced Thompson Sampling algorithms for contextual bandits with auxiliary safety constraints \cite{daulton2019thompson}, and investigated adaptive sequential experiments with dynamic auxiliary information  to improve the robustness of decision-making \cite{gur2022adaptive}. 
In our work, we will take advantage of auxiliary information to overcome the challenge of missing outcomes. 

Addressing the problem of missing outcomes is both practically relevant and theoretically challenging.
In applications such as healthcare, education, and e-commerce, accounting for missing data could lead to better treatment policies, more personalized learning experiences, and more effective product recommendations, potentially affecting millions of individuals.
In this paper, we undertake the first formal study of multi-armed bandits with missing outcomes and provide tailored algorithms that explicitly handle different types of missingness.
Our main contributions are two-fold.
First, we provide an analysis of the impact of missing outcomes on achievable regret (the loss of optimality).
Second, we introduce provably good upper confidence bound (UCB) algorithms that are tailored to handle both missing at random and missing not at random mechanisms.
Our algorithms are designed to adjust reward estimates based on the observed data and the missingness mechanism, ensuring unbiased estimation.
Finally, we extend our analysis to settings where not only outcomes but also mediators (e.g., users providing feedback) are prone to missingness, to further broaden the applicability of our approach.

The remainder of this paper is structured as follows.
In \Cref{sec:prb}, we review the relevant background and formalize the problem of multi-armed bandits with missing outcomes.
In \Cref{sec:missingoutcome} we present our algorithms in the settings of MCAR, MAR, and MNAR, respectively.
Additionally, we provide the corresponding achievable regret lower bounds.
The technical proofs are postponed to \Cref{apx:proofs}.
In \Cref{sec:missingmediator}, we extend our approach and present algorithms for the case when the mediator is also prone to missingness.
A discussion of the limitations of our work and our concluding remarks appear in \Cref{sec:conc}.

\section{Formalization and Problem Setup}\label{sec:prb}
We begin by reviewing the classic multi-armed bandit (MAB) setup and then extend it to incorporate missing outcomes.
The MAB problem involves an agent (decision-maker) who interacts with an environment over a sequence of $T$ time steps.
At each time step $t\in\{1,\dots,T\}$, the agent pulls an arm $a_t$ from a set of $n$ available actions indexed by $\mathcal{A}=\{1,\dots,n\}$.
Upon pulling this arm, the agent receives a stochastic reward $Y_t\in\mathcal{Y}$ drawn from a fixed (but unknown) probability distribution associated with arm $a_t$.
The goal of the agent is to minimize the \emph{cumulative regret} over the time horizon $T$, which is defined as the cumulative difference between the rewards of the optimal arm and the chosen arms.
Specifically, let $\mu_a \coloneqq \ex{Y \mid A=a}$ for every $a\in\mathcal{A}$.
The optimal arm, denoted by $a^*$, is the arm that maximizes the expected reward, i.e., 
\(a^*\coloneqq \argmax_{a\in\mathcal{A}}\mu_a.\)
The regret at time $t$ is  defined as
\(R_t \coloneqq \mu_{a^*}-\ex{Y\mid A=a_t}\),
and the cumulative regret over $T$ rounds, denoted by $R_T$, is the sum of the latter instantaneous regrets over the horizon T:
\begin{equation}
\begin{split}
    R_T &\coloneqq \sum_{t=1}^T(\mu_{a^*}-\ex{Y\mid A=a_t}) = T\mu_{a^*} - \sum_{t=1}^T \ex{Y\mid A=a_t}.
\end{split}
\end{equation}
In the classical setting, it is assumed that after pulling an arm $a_t$, the agent always observes the true reward $Y_t$ without any missingness.

We extend the classic MAB model to accommodate missingness.
We assume that pulling each arm $a_t\in\mathcal{A}$, draws a stochastic tuple 
$\big(Y_t, \oo_t, M_t, \mm_t \big)$ 
 from a fixed but unknown probability distribution associated with arm $a_t$. In this tuple,
$Y_t\in\mathcal{Y}$ represents the true reward (as before), whereas $\oo_t\in\{0,1\}$ is an indicator denoting whether this reward is observed. $M_t\in\mathcal{M}$ is a possible mediator or an auxiliary variable\footnote{The introduction of this auxiliary variable is without loss of generality, as it can be simply set to $M\equiv 0$, i.e., a degenerate variable that carries no information.}, with $\mm_t\in\{0,1\}$ indicating whether this auxiliary variable is observed.
For example in online recommendations, auxiliary information could include metrics such as the time a user spends on a webpage before navigating away, or other data points gathered from browser cookies, such as past browsing behavior, device type, or location.
The agent has access to the `observed' tuple $\big(Y^o_t, \oo_t, M^o_t, \mm_t \big)$,
where the observed values $Y^o_t$ and $M^o_t$ are defined as follows:
\begin{equation}\label{eq:consistency}
\begin{split}
    Y^o_t=\begin{cases}
        Y_t; &\text{ if } \oo_t =1,\\
        ?; &\text{ o.w.}
    \end{cases},
    M^o_t=\begin{cases}
        M_t; &\text{ if } \mm_t =1,\\
        ?; &\text{ o.w.}
    \end{cases},
\end{split}
\end{equation}
where $?$ denotes a missing value.
We define $\mu_{a}$ as the expected value of $Y_t$ given $A_t=a$ as before, with the crucial difference that samples of $Y_t$ are missing when $\oo_t=0$.

Clearly, without imposing further structure, it is not possible to construct unbiased estimators for the expected rewards of each arm. 
In fact, these expectations are not `identifiable,' meaning that they are not uniquely determinable functionals of the probability measure over observable variables.
In what follows, we begin with the case where the mediator is fully observed ($\mm_t=1$ with probability $1$).
We first consider the case where the missingness mechanism of the outcome is independent of everything else.
Subsequently, we analyze the more realistic cases where this missingness is correlated with the missing outcome $Y_t$.
Later in \Cref{sec:missingmediator} we extend our findings further to the case where even the mediator is prone to missingness.

\section{MAB with Missing Outcome}\label{sec:missingoutcome}
Throughout, we assume that the outcomes are not `always missing.'
\begin{assumption}[Positivity]\label{as:pos}
    For every action $a\in\mathcal{A}$,
    \(\mathbb{P}\big(\oo_t=1\mid M_t, A_t=a\big)> 0\).
    Moreover, $\mathbb{P}(M_t\mid A_t)$ is positive everywhere\footnote{With sufficient caution, the second part of this assumption could be omitted. 
    However, we include it here for the sake of simplicity in the presentation.}.
\end{assumption}
Assumption \ref{as:pos} is reasonable as otherwise there exists an arm for which the agent observes no reward samples.
For the rest of this section, we assume that the auxiliary variable $M_t$ is always observed.


    

\subsection{Missing Completely At Random ({{\small MCAR}})}
We begin with the case where the outcome missingness mechanism is independent of the other variables (including the outcome itself).
\begin{figure}[t]
    \begin{center}
        
    \begin{subfigure}[b]{0.23\textwidth}
        \centering
        \begin{tikzpicture}
            \tikzset{
                solid/.style={circle, fill, inner sep=1.5pt},
                every path/.style={thick, >={Latex[round]}}
                }
            \node[solid] (A) at (-1.2,0) {};
            \node[solid] (M) at (0,0) {};
            \node[solid] (Y) at (1.2,0) {};
            \node[solid] (RY) at (0.1,1) {};
            
            \draw[->, bend right=30] (A) to (Y);
            \draw[->] (A) to (M);
            \draw[->] (M) to (Y);

            \node[left=1pt of A] {A};
            \node[above=1pt of M] {M};
            \node[right=1pt of Y] {Y};
            \node[above=1pt of RY] {$\oo$};
        \end{tikzpicture}
        \caption{MCAR}
        \label{fig:mcar}
    \end{subfigure}
    \begin{subfigure}[b]{0.23\textwidth}
        \centering
        \begin{tikzpicture}
            \tikzset{
                solid/.style={circle, fill, inner sep=1.5pt},
                every path/.style={thick, >={Latex[round]}}
                }
            \node[solid] (A) at (-1.2,0) {};
            \node[solid] (M) at (0,0) {};
            \node[solid] (Y) at (1.2,0) {};
            \node[solid] (RY) at (0.1,1) {};
            
            \draw[->, bend right=30] (A) to (Y);
            \draw[->] (A) to (M);
            \draw[->] (M) to (Y);
            \draw[->] (M) to (RY);
            \draw[->] (A) to (RY);

            \node[left=1pt of A] {A};
            \node[above right=1pt of M] {M};
            \node[right=1pt of Y] {Y};
            \node[above=1pt of RY] {$\oo$};
        \end{tikzpicture}
        \caption{MAR (i)}
        \label{fig:mar1}
    \end{subfigure}
    \begin{subfigure}[b]{0.23\textwidth}
        \centering
        \begin{tikzpicture}
            \tikzset{
                solid/.style={circle, fill, inner sep=1.5pt},
                every path/.style={thick, >={Latex[round]}}
                }
            \node[solid] (A) at (-1.2,0) {};
            \node[solid] (M) at (1.2,0) {};
            \node[solid] (Y) at (0,0) {};
            \node[solid] (RY) at (0.1,1) {};
            
            \draw[->, bend right=30] (A) to (M);
            \draw[->] (A) to (Y);
            \draw[->] (Y) to (M);
            \draw[->] (M) to (RY);
            \draw[->] (A) to (RY);

            \node[left=1pt of A] {A};
            \node[ right=1pt of M] {M};
            \node[above=1pt of Y] {Y};
            \node[above=1pt of RY] {$\oo$};
        \end{tikzpicture}
        \caption{MAR (ii)}
        \label{fig:mar2}
    \end{subfigure}
    \begin{subfigure}[b]{0.23\textwidth}
        \centering
        \begin{tikzpicture}
            \tikzset{
                solid/.style={circle, fill, inner sep=1.5pt},
                every path/.style={thick, >={Latex[round]}}
                }
            \node[solid] (A) at (-1.2,0) {};
            \node[solid] (M) at (0,0) {};
            \node[solid] (Y) at (1.2,0) {};
            \node[solid] (RY) at (0.1,1) {};
            
            \draw[->, bend right=30] (A) to (Y);
            \draw[->] (A) to (M);
            \draw[->] (M) to (Y);
            \draw[->] (A) to (RY);
            \draw[->] (Y) to (RY);

            \node[left=1pt of A] {A};
            \node[above right=1pt of M] {M};
            \node[right=1pt of Y] {Y};
            \node[above=1pt of RY] {$\oo$};
        \end{tikzpicture}
        \caption{MNAR}
        \label{fig:mnar}
    \end{subfigure}
    
    \caption{Graphical representations of the missing data mechanisms considered in this paper. }
    \label{fig:missing mechansim}
        \end{center}

\end{figure}
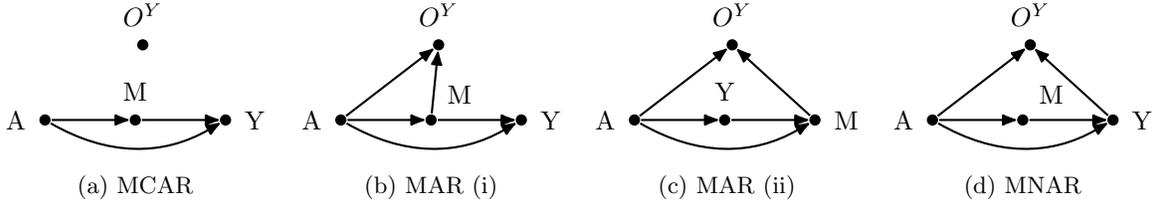

\begin{assumption}[MCAR]
    The outcome is missing completely at random.
    That is, \(\oo_t\indep(A_t, Y_t, M_t )\) for $t\in\{1,\dots,T\}$.
\end{assumption}
This assumption holds, for instance, when data gets erased by say an independent mechanism such as a power outage.
The graph of Figure~\ref{fig:mcar} represents this missingness mechanism, whereby the missingness indicator $\oo$ is an isolated node.
As there is no information conveyed by the missingness indicator, the missing chunk of the data can be discarded without any need for extra care.
As such, the classic upper confidence bound (UCB) algorithms are expected to achieve (near-) optimal regret.
We formalize these claims through the next two propositions.
For the sake of completeness, we have included the adapted UCB algorithm (Alg.~\ref{alg:mcar_algorithm}) for this scenario in Appendix \ref{apx:alg}.



\begin{restatable}{theorem}{theomcarupper} \label{theo:mcar_upper}\text{(MCAR regret guarantee)} Under Assumption \ref{as:pos}, for every \( \alpha > 1 \), the cumulative regret of the adapted UCB (Alg.~\ref{alg:mcar_algorithm}) is bounded as follows:
\[
\mathbb{E}[R_T] = O\left(\sqrt{\frac{ \alpha n T \log(T)}{\gamma}}\right).
\]
\end{restatable}

The proof of \Cref{theo:mcar_upper}, which provides a regret bound similar to that of the classic UCB algorithm, but adapted to our setting, is included in Appendix \ref{apx:proofs} provided in the supplementary.
The following result indicates that this regret bound is (near-)optimal.
\begin{restatable}{theorem}{theomcarlower}\label{theo:mcar_lower}
\text{(Minimax lower bound for MCAR)}  
For any policy \( \pi \), there exists an MCAR instance \( \nu \) s.t.
\[
\mathbb{E}[R_T(\pi, \nu)] = \Omega\left( \sqrt{\frac{nT}{\gamma}} \right),
\]
where \( \mathbb{E}[R_T(\pi, \nu)] \) represents the expected regret of policy \( \pi \) in instance \( \nu \).
\end{restatable}

See \Cref{apx:proofs} for the proof of \Cref{theo:mcar_lower} as well as the rest of the results of this paper.
\subsection{Missing At Random (MAR)}\label{subsec:mar}
We now focus on more realistic settings where the missingness mechanism provides information about the missing outcomes. 
This is the case for instance when the unsatisfied customers are more likely to leave comments on an online platform, or in health applications, the patients with severe side effects are more likely to drop out of the study.
We first consider the case when missingness is at random, i.e., independent of $Y$ given $M$ and $A$. 
The graphs of Figure~\ref{fig:mar1} and Figure~\ref{fig:mar2} illustrate two possible representations of the MAR mechanism, under which Assumption~\ref{ass:mar_assumption} holds.

\begin{assumption} [MAR]\label{ass:mar_assumption}
\(\oo_t\indep Y_t \mid (A_t, M_t )\) for $t\in\{1,\dots,T\}$.
\end{assumption}
Under \cref{ass:mar_assumption}, 
the expected reward is identifiable as follows:
\begin{equation}\label{eq:idmar}
    \begin{split}
        \mu_a=\mathbb{E}[Y_t&\mid A_t=a] = \mathbb{E}\big[ \mathbb{E}[Y_t \mid M, a] \mid A_t=a\big]\\
        &\overset{(a)}{=} \mathbb{E}\big[  \mathbb{E}[Y_t \mid M, a, \oo_t = 1]\mid A_t=a\big]\\
        &\overset{(b)}{=} \mathbb{E}\big[  \mathbb{E}[Y^o_t \mid M, a, \oo_t = 1]\mid A_t=a\big],
    \end{split}
\end{equation}
where $(a)$ and $(b)$ follow from \Cref{ass:mar_assumption} and \Cref{eq:consistency}, respectively.
Accordingly, we will use the following estimator for $\mu_a$:
\begin{equation}\label{eq:marest}\hat{\mu}_a = \frac{1}{\vert T_a\vert}\sum_{t\in T_a}\big(\sum_{m\in\mathcal{M}}\frac{\mathbbm{1}\{M_{t}=m\}}{\vert T_{m,a,o}\vert}\sum_{t'\in T_{m,a,o}} Y^o_{t'}\big),\end{equation}
where $T_{a}, T_{m,a,o}\subseteq\{1,\dots,T\}$ are the set of iterations where $A_t=a$, and iterations where $A_t=a$, $M_t=m$ and $\oo_t=1$, respectively.
In what follows, for brevity, we define $p_{m,a}\coloneqq\mathbb{P}(M_t=m\mid A_t=a)$.
We first present an algorithm for minimizing regret when the conditional probabilities $p_{m,a}$ are known.
We then adapt our algorithm to the case where these probabilities are unknown.
Recall that \( n=\vert\mathcal{A}\vert \) is the number of arms. 
We assume that \( \mathbb{E}[Y_t \mid m, a] \in [0, 1] \) for all arms and that the reward \( Y_t \) is sub-Gaussian.
\Cref{alg:mar_algorithm} presents the pseudo-code for the first case.
The algorithm is based on UCB, with the difference that at the beginning, the agent pulls every arm
\( \log(T)^2 \) times uniformly.
At the subsequent rounds, both the expected rewards and the associated confidence bands are estimated based on \Cref{eq:marest}.
In order to present the regret bounds, we need the following definitions.
Let \( P_a = \sum_{m \in \mathcal{M}} \frac{p_{m,a}}{\gamma_{m,a}} \) where \( \gamma_{m, a} = \mathbb{P}(O^Y = 1 \mid m, a)\).
Further, define \( S\) and $H$ as the arithmetic mean, and the harmonic mean of the \( P_a \) values, respectively: 
\[ S \coloneqq \frac{1}{\vert \mathcal{A}\vert} \sum\limits_{a\in\mathcal{A}} P_a,\quad H\coloneqq \frac{\vert \mathcal{A}\vert}{\sum_{a\in\mathcal{A}}\frac{1}{P_a}}. \]
\begin{restatable}{theorem}{theomarupperfirst} \label{theo:mar_upper_1}
\text{(Regret guarantee for Algorithm~\ref{alg:mar_algorithm})}  
For every \( \alpha > 1 \), the following regret bound holds for sufficiently large \( T \):
\[
\mathbb{E}[R_T] = O\left( \sqrt{\alpha T \log(T) n S} \right).
\]
\end{restatable}

Next, we show that \Cref{alg:mar_algorithm} can be adapted to the case where the conditional probabilities $p_{m,a}$ are not known and must be estimated -- see Algorithm~\ref{alg:mar_algorithm2}.
The following theorem shows that this algorithm achieves the same regret bound as \Cref{alg:mar_algorithm}, i.e., the estimation of $p_{m,a}$ does not affect the cumulative regret.

\begin{restatable}{theorem}{theomaruppersecond} \label{theo:mar_upper_2}
\text{(Regret guarantee for Algorithm~\ref{alg:mar_algorithm2})}  
For every \( \alpha > 1 \), the following regret bound holds for sufficiently large \( T \):
\[
\mathbb{E}[R_T] = O\left( \sqrt{\alpha T \log(T) n S} \right).
\]
\end{restatable}
    It is noteworthy that these regret bounds do not depend on the cardinality of the mediator ($\mathcal{M}$).
The following theorem provides the minimax lower bound, demonstrating near-optimality of Algorithms \ref{alg:mar_algorithm} and \ref{alg:mar_algorithm2}.

\begin{restatable}{theorem}{theomarlower} \label{theo:mar_lower_1}
\text{(Minimax lower bound for MAR)}  
For any policy $\pi$, there exists a MAR instance $\nu$ such that:
\[
\mathbb{E}[R_T(\pi, \nu)] = \Omega\left( \sqrt{T n H}\right).
\]
\end{restatable}
Note that when $\gamma_{m,a}$ values are identical (and equal to $\gamma$) then $S$ and $H$ coincide.
Further, 
the upper and lower bounds in this case match those of MCAR.

A special case of the MAR environment (depicted in Figure~\ref{fig:semi_mcar}) pertains to when there is no mediator. In this case,  \Cref{ass:mar_assumption} reduces to the following:
\begin{assumption}\label{as:mar2}
    $\oo_t\indep Y_t\mid A_t$ for all $t\in\{1,\dots,T\}$.
\end{assumption}
Theorems~\ref{theo:mar_upper_1} and \ref{theo:mar_lower_1}
with a degenerate mediator ($\vert\mathcal{M}\vert=1$) imply the following corollary.
\begin{corollary}
    Under \Cref{as:mar2} \Cref{alg:mar_algorithm2} induces cumulative regret
    \(\
        \mathbb{E}[R_T] = O\left( \sqrt{\alpha T \log(T) n S} \right)
    \)
    and the cumulative regret of any policy is lower bounded by
    \(
        \mathbb{E}[R_T] = \Omega\left( \sqrt{\alpha T \log(T) n H} \right),
    \)
    where \( S = \frac{\sum\limits_{a} \frac{1}{\gamma_a}}{n} \) and \( H = \frac{n}{\sum\limits_{a} \gamma_a} \).
\end{corollary}

\begin{figure}[t]
        \centering
        \begin{tikzpicture}
            \tikzset{
                solid/.style={circle, fill, inner sep=1.5pt},
                every path/.style={thick, >={Latex[round]}}
                }
            \node[solid] (A) at (-1.2,0) {};
            \node[solid] (Y) at (1.2,0) {};
            \node[solid] (RY) at (0.1,0.7) {};
            
            \draw[->] (A) to (Y);
            \draw[->] (A) to (RY);            
    
            \node[left=1pt of A] {A};
            \node[right=1pt of Y] {Y};
            \node[above right=-8pt and 1pt of RY] {$\oo$};
        \end{tikzpicture}
        \caption{Special case of MAR.}
        \label{fig:semi_mcar}
\end{figure}
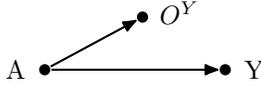

\paragraph{Discussion 1.}\label{discussion1} We used estimators that explicitly use the mediator values in this section.
As we pointed out earlier, the size of $\mathcal{M}$ (the alphabet of $M$) does not affect the regret bounds.
However, one might wonder whether the use of the mediator can be avoided, resulting in simpler algorithms and/or estimation schemes.
We show next that any such algorithm can induce linear regret in the worst case.
As a corollary, this result implies that naïvely employing the classical UCB algorithm also induces linear regret.
\begin{restatable}{theorem}{thmignoremed}\label{thm:ignoremed}
    For any mediator-agnostic  policy $\pi$ (a policy that does not have access to mediator values), there exists a MAR instance $\nu$ which satisfies \Cref{ass:mar_assumption} and its regret grows linearly
    \[
        \ex{R_T(\pi,\nu)}=\Omega(T).
    \]
\end{restatable}

\paragraph{Discussion 2.}\label{discussion2} The expected reward $\mu_a$ can also be estimated using a Horvitz-Thompson (HT) type estimator \cite{HorvitzThompson1952}.
Specifically, the conditional expectation terms in \Cref{eq:idmar} can be expressed as follows:
\[
\begin{split}
    \ex{Y^o_t\mid m,a, \oo_t&=1}=
    \ex{\frac{Y^o_t \mathbbm{1}\{M_t=m, \oo_t=1\}}{p_{m,a}\gamma_{m,a}}\mid A_t=a},
\end{split}
\]
and after plugging it into Eq.~\eqref{eq:idmar}, 
\begin{equation}\label{eq:ht}
    \mu_a = \ex{\sum_{m\in\mathcal{M}}
        \frac{Y^o_t\mathbbm{1}\{M_t=m, \oo_t=1\}}{\gamma_{m,a}}
        \mid A_t=a
    }.
\end{equation}
An estimator based on the latter does not require estimating the conditional outcome means (in contrast to Eq.~\ref{eq:idmar}), but it rather requires modeling the missingness probabilities $\gamma_{m,a}$.
Using such an estimator is particularly beneficial when the missingness probabilities are known in advance, or a parametric model can be justified.
However, if the missingness probabilities are small or estimated imprecisely, the HT estimator can exhibit high variance, leading to instability.
One can take a step further and construct augmented inverse propensity weighted (AIPW) estimators for $\mu_a$:
\begin{equation}\label{eq:dr}
\begin{split}
    \mu_a &= \mathbb{E}\big[
        \sum_{m\in\mathcal{M}}\frac{\mathbbm{1}\{M_t=m\}}{\gamma_{m,a}}\big(
        Y^o_t\mathbbm{1}\{\oo_t=1\}
        -
        (\mathbbm{1}\{\oo_t=1\}-\gamma_{m,a})\ex{Y^o_t\mid m, a, \oo_t=1}
        \big)
        \mid A_t=a
    \big],
\end{split}
\end{equation}
which is doubly robust (DR) in the sense that it is consistent as long as either the missingness probabilities $\gamma_{m,a}$ or the conditional outcome means $\ex{Y^o_t\mid m, a, \oo_t=1}$ (but not necessarily both) are correctly specified.
We prove this claim formally in Appendix \ref{apx:proofs} for the sake of completeness.
In this paper, we consider discrete-valued mediators, and we estimate all the quantities of interest through empirical means. 
Therefore, all three estimators (outcome-based, HT, and DR) coincide.
However, the HT and DR estimators can prove beneficial for extending our approach to incorporate continuous mediators, or in problems with high-dimensional actions and/or mediators where (semi)parametric models can help improve estimation efficiency.
\subsection{Missing Not At Random (MNAR)}\label{sec:mnar}
Finally, we consider the case where the missingness mechanism directly depends on the outcome value \( Y \). 
Here, we follow the identification strategy of \cite{zuo2024mediation} for MNAR.
However, we are interested only in identifying the expected rewards, rather than conducting mediation analysis.
We begin with the following assumption.
\begin{assumption}[MNAR]\label{ass:mnar_assumption}
\(
    \oo_t \indep M_t \mid (A_t, Y_t)
    \) for \( t \in \{1, \dots, T\} \).
\end{assumption}
In other words, the missingness is independent of the mediator when conditioned on the action and the actual outcome.
\Cref{fig:mnar} graphically represents this scenario.
This situation commonly arises in environments where the reward is missing due to its value. 
For example, if the outcome of interest is the income of an individual, they may not be inclined to report it if the value is too high or too low.

We further make the following assumption.
\begin{assumption}[Completeness]\label{as:complete}
    The distribution \( \mathbb{P}(M, Y, O^Y = 1 \mid a) \) is complete in \( M \), that is, for any $a\in\mathcal{A}$, and for any function $g:\mathcal{Y}\to\mathbb{R}$,
    \[
    \int_{y \in \mathcal{Y}} \mathbb{P}(M, Y=y, O^Y = 1 \mid a)g(y) \, dy = 0
    \]
    implies that \( g(Y) = 0 \) with probability one.
\end{assumption}


Below we show how $\mu_a$ is identified under these assumptions.
The identification strategy outlined here is akin to \cite{zuo2024mediation}.
%
\begin{align*}
    &\mathbb{P}(m, O^Y = 0 \mid a) 
    = \int_{y \in \mathcal{Y}} \mathbb{P}(m, y, O^Y = 0 \mid a) \, dy \\
    &\overset{(a)}{=} \int_{y \in \mathcal{Y}} \mathbb{P}(m, y, O^Y = 1 \mid a) \frac{\mathbb{P}(O^Y = 0 \mid y, a,m)}{\mathbb{P}(O^Y = 1 \mid y, a,m)} \, dy\\
    &\overset{(b)}{=} \int_{y \in \mathcal{Y}} \mathbb{P}(m, y, O^Y = 1 \mid a) \frac{\mathbb{P}(O^Y = 0 \mid y, a)}{\mathbb{P}(O^Y = 1 \mid y, a)} \, dy,
\end{align*}
where $(a)$ and $(b)$ follow from Bayes' rule and \Cref{ass:mnar_assumption}, respectively.
Since \( \mathbb{P}(M, Y, O^Y = 1 \mid a) \) is complete in \( M \), solving this integral equation uniquely determines the inverse odds ratio \( \mathrm{OR}_{y,a}=\frac{\mathbb{P}(O^Y = 0 \mid y, a)}{\mathbb{P}(O^Y = 1 \mid y, a)} \), allowing us to identify \( \mathbb{P}(y\mid a) \) as follows:
\begin{equation}\label{eq:idmnar}
\begin{split}
\mathbb{P}(y\mid a) &=\! \sum_{m\in\mathcal{M}}
\mathbb{P}(y, m \mid a) = 
\sum_{m\in\mathcal{M}}\frac{\mathbb{P}(y, m \mid O^Y = 1, a)}{\mathbb{P}(O^Y = 1 \mid y, a)}
\\&=\! \sum_{m\in\mathcal{M}} (1+\mathrm{OR}_{y,a})\mathbb{P}(y, m \mid O^Y = 1, a).
\end{split}
\end{equation}
Finally, \( \mu_a=\mathbb{E}[Y_t\mid A_t=a] \) is identified as
\(
\mu_a = \int_{y \in \mathcal{Y}} y \mathbb{P}(y \mid a) \, dy.
\)

In the remainder of this section, 
we assume \( Y \) is discrete with \( |\mathcal{Y}| = L \), 
and the outcomes are normalized so that \( \sum_{y \in \mathcal{Y}} |y| = 1 \).
Define \( K=\vert\mathcal{M}\vert\), and \( \Theta_a = [\mathbb{P}(m, y, O^Y = 1\mid a)]_{K \times L} \).
Additionally, we assume that these matrices are not ill-conditioned. 

\begin{restatable}{assumption}{asbounded}[Bounded condition number]\label{ass:mnar_K_assumption}
For each arm \( a \in \mathcal{A} \), the condition number of the matrix \( \Theta_a \) is bounded by:
\[
    \kappa(\Theta_a) \leq C_a,
\]
where \( \kappa(\Theta_a) \) denotes the condition number of \( \Theta_a \) with respect to $\infty$-norm, defined as 
\(
    \kappa(\Theta_a) = \lVert \Theta_a \rVert_\infty \lVert \Theta_a^{\dagger} \rVert_\infty,
\) with $\Theta_a^{\dagger}$ being the pseudo-inverse of $\Theta_a$.
\end{restatable}

We present Algorithm~\ref{alg:mnar_algorithm} for minimizing cumulative regret under the MNAR assumptions. 
The key intuition behind this algorithm is to construct an estimator based on Eq.~\eqref{eq:idmnar} and build upper confidence bounds under \Cref{ass:mnar_K_assumption}.
In order to present the regret bound of this algorithm, 
define \( p_{y, a} = \mathbb{P}(Y = y \mid A = a)\), and \(\gamma_{y, a} = \mathbb{P}(O^Y = 1 \mid Y = y, A = a)\).

\begin{restatable}{theorem}{theomnarupper} \label{theo:mnar_upper}
\text{(Regret guarantee for Algorithm~\ref{alg:mnar_algorithm})}  
For every \( \alpha > 1 \), the following regret bound holds for sufficiently large \( T \):
\[
\mathbb{E}[R_T] = O\left( \sqrt{\alpha T \log(T) \sum\limits_{a} S_a^2} \right),
\]
with \( S_a \!=\! \max \{ \frac{L C_a}{\gamma_a \lVert \Theta_a \rVert_\infty },\frac{K}{\gamma_a \sqrt{\sum\limits_{y \in \mathbb{Y}} p_{y, a} \gamma_{y, a}}}
\} \),\( \gamma_a \!=\! \min\limits_{y} \gamma_{y, a} \).
\end{restatable}

\begin{remark}
With \( \gamma_\mathrm{min} = \min\limits_{y, a} \gamma_{y, a} \) and \( \kappa_\mathrm{max} = \max\limits_{a} \frac{C_a}{\lVert \Theta_a \rVert_\infty } \), \Cref{theo:mnar_upper} implies the following bound:
\[
\mathbb{E}[R_T] = O\left( \sqrt{\alpha T \log(T) N 
\max \{ \frac{L \kappa_\mathrm{max}}{\gamma_\mathrm{min}},\frac{K}{
\gamma_{\mathrm{min}}^{3/2}}}\}^2 \right).
\]

\end{remark}


\section{MAB with Missing Outcome and Mediator}\label{sec:missingmediator}
So far we considered cases where the mediator was fully observable.
We now discuss how our results extend to scenarios involving missing data in both \( Y \) and \( M \).
In this section, we assume that the outcome is MAR, and discuss the cases where the mediator is MAR and MNAR separately.
For the case where both outcome and mediator are MNAR, refer to \Cref{apx:Missing M}.
We begin by outlining each scenario, providing identification schemes and estimators for \( \mu_a \). 
The corresponding algorithms, theoretical results, and proofs are postponed to Appendix~\ref{apx:Missing M}.
\begin{center}
    \begin{figure}[t]
        \begin{center}

    \begin{subfigure}[b]{0.23\textwidth}
                    \begin{tikzpicture}
            \tikzset{
                solid/.style={circle, fill, inner sep=1.5pt},
                every path/.style={thick, >={Latex[round]}}
                }
            \node[solid] (A) at (-1.2,0) {};
            \node[solid] (M) at (0,0) {};
            \node[solid] (Y) at (1.2,0) {};
            \node[solid] (RY) at (1.2,1.2) {};
            \node[solid] (RM) at (0.1,1.2) {};

            \draw[->, bend right=30] (A) to (Y);
            \draw[->] (A) to (M);
            \draw[->] (M) to (Y);
            \draw[->] (M) to (RY);
            \draw[->] (A) to (RY);
            \draw[->] (A) to (RM);
            \draw[->] (RM) to (RY);

            \node[left=1pt of A] {A};
            \node[above=0.5pt of M] {M};
            \node[right=1pt of Y] {Y};
            \node[above=1pt of RY] {$\oo$};
            \node[above=1pt of RM] {$O^M$};
            
        \end{tikzpicture}
        \caption{MAR}
        \label{fig:m,y mar}
    \end{subfigure}
    \begin{subfigure}[b]{0.23\textwidth}
        \centering
        \begin{tikzpicture}
            \tikzset{
                solid/.style={circle, fill, inner sep=1.5pt},
                every path/.style={thick, >={Latex[round]}}
                }
            \node[solid] (A) at (-1.2,0) {};
            \node[solid] (M) at (0,0) {};
            \node[solid] (Y) at (1.2,0) {};
            \node[solid] (RY) at (1.2,1.2) {};
            \node[solid] (RM) at (0.1,1.2) {};

            \draw[->, bend right=30] (A) to (Y);
            \draw[->] (A) to (M);
            \draw[->] (M) to (Y);
            \draw[->] (A) to (RY);
            \draw[->] (M) to (RY);
            \draw[->] (M) to (RM);
            \draw[->] (A) to (RM);
            
            \node[left=1pt of A] {A};
            \node[above left=-2.9pt of M] {M};
            \node[right=1pt of Y] {Y};
            \node[above=1pt of RY] {$\oo$};
            \node[above=1pt of RM] {$O^M$};
            
        \end{tikzpicture}
        \caption{MNAR}
        \label{fig:m,y mnar}
    \end{subfigure}
       
    \caption{Graphical representations of the missing data mechanisms with missing outcome and mediator. }
    
    \end{center}

    \label{fig:missing mechansim2}
\end{figure}
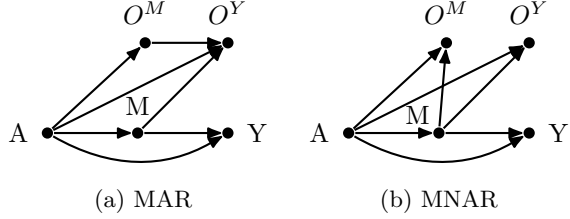

\end{center}
Throughout this section, we work under \Cref{ass:mar_assumption}.

\subsection{MAR}
Since the mediator values are missing, neither the conditional outcome means nor the probabilities $p_{m,a}$ are identifiable, and we require further structure to make progress.
One such structure arises
when the missingness in the mediator ($\mm$) can be assumed to be at random, i.e., $\mm_t\indep(M_t,Y_t,\oo_t)\mid A_t$.
This assumption is valid for instance when the missingness mechanism for the mediator depends only on the action. 
A less stringent alternative can be formalized as:
\begin{assumption}\label{ass:mar_my_assumption} 
    \(
    \mm_t \indep M_t \mid A_t,
    \) and $\mm_t\indep Y_t\mid (A_t,M_t,\oo_t)$ for all $t\in\{1,\dots,T\} $.
\end{assumption}
See Fig.~\ref{fig:m,y mar} for a graph representation satisfying Assumptions \ref{ass:mar_assumption} and \ref{ass:mar_my_assumption}.
Under these two assumptions, analogous to Eq.~\eqref{eq:idmar}, $\mu_a$ can be identified as follows.
\[
    \begin{split}
        &\mu_a
        = \mathbb{E}\big[  \mathbb{E}[Y_t \mid M, a, \oo_t = 1]\mid A_t=a\big]\\
        &= \mathbb{E}\big[  \mathbb{E}[Y^o_t \mid M, a, \oo_t = 1,\mm_t=1]\mid A_t=a,\mm_t=1\big],
        \end{split}
\]
where the second equation is due to \Cref{ass:mar_my_assumption}.

\subsection{MNAR}
When the mediator is missing not at random, stronger assumptions are necessary to identify the expected rewards. 
Analogous to Section \ref{sec:mnar}, we will use a completeness assumption.
Here too, the identification strategy follows the approach of \cite{zuo2024mediation}.
\begin{assumption}\label{ass:mnar_my_assumption} 
    $Y_t$, $\oo_t$ and $\mm_t$ are mutually independent conditioned on $(A_t,M_t)$ for all $t\in\{1,\dots,T\}$.
\end{assumption}
\begin{assumption}\label{as:complete2}
    For every $a\in\mathcal{A},m\in\mathcal{M}$, $\mathbb{P}(M = m, Y, O^M = 1, O^Y = 1 \mid a)$ is complete in $Y$.
    That is, for any function $g:\mathcal{Y}\to \mathbb{R}$,
    \[\int_{\mathcal{Y}}\mathbb{P}(M = m, Y=y, O^M = 1, O^Y = 1 \mid a) g(y)dy=0\]
    implies $g(Y)=0$ with probability one.
\end{assumption}
Under Assumption~\ref{ass:mnar_my_assumption}, \(\mu_a\) can be expressed as:
\begin{align*}
    \mu_a &
    = \sum_{m \in \mathbb{M}}  \mathbb{E}[Y \mid a, m]p_{m,a} \\
    &= \sum_{m \in \mathbb{M}} \mathbb{E}[Y \mid a, m, O^Y = 1, O^M = 1]p_{m,a}.
\end{align*}

To proceed, we need to identify \( p_{m,a}=\mathbb{P}(M = m \mid A = a) \). 
This is achieved through \Cref{as:complete2}:
\begin{align*}
    &\mathbb{P}(Y = y, O^M = 0, O^Y = 1 \mid a) \\
    &= \sum_{m \in \mathbb{M}} \mathbb{P}(M = m, Y = y, O^M = 0, O^Y = 1 \mid a) \\
    &= \sum_{m \in \mathbb{M}} \mathbb{P}(M = m, Y = y, O^M = 1, O^Y = 1 \mid a) \\
    &\quad \times \frac{\mathbb{P}(O^M = 0 \mid M = m, A = a)}{\mathbb{P}(O^M = 1 \mid M = m, A = a)},
\end{align*}
where we used \Cref{ass:mnar_my_assumption} in the last equation.
By \cref{as:complete2}, the inverse odds ratios $\mathrm{OR}_{m,a}=\frac{\mathbb{P}(O^M = 0 \mid m, a)}{\mathbb{P}(O^M = 1 \mid m, a)}$ can be uniquely determined.
Finally,
\[
\begin{split}
    p_{m,a} &= \frac{\mathbb{P}(M = m, O^M = 1 \mid A = a)}{\mathbb{P}(O^M = 1 \mid A = a, M = m)},\\
    &=(1+\mathrm{OR}_{m,a})\mathbb{P}(M = m, O^M = 1 \mid A = a).
\end{split}
\] 
We use a two-step estimation process, whereby in the first step, $\hat{p}_{m,a}$ is estimated, and in the second step, the expected reward is estimated as
\[
\hat{\mu}_a = \sum_{m \in \mathcal{M}} \hat{p}_{m, a} \hat{\mu}_{m, a}
\]
where \( \hat{\mu}_{m, a} \) is the empirical mean of the samples \( Y_t \), obtained after pulling arm \( a \), conditioned on \( M_t=m \) with both \( O^M = 1 \) and \( O^Y = 1 \).
Here, we require \( Y_t \) to be finite-valued, analogous to Section \ref{sec:missingoutcome}.

\section{Empirical Evaluation}
\label{sec:empirical-evaluation}

Here, we provide an empirical evaluation of our MAB algorithms across different missing data scenarios -- MCAR, MAR(i,ii), and MNAR. 
All our simulations were run on Google Colab\footnote{\hyperlink{https://colab.google}{https://colab.google}} with Intel Xeon CPUs. The python codes for reproducing our experimental results are publicly available at the repository\footnote{\url{https://github.com/ilia-mahrooghi/Multi-armed-Bandits-with-Missing-Outcome}}.
We model the MAB environment in all the aforementioned settings with \( n = 10 \) arms.
More comprehensive simulation results are provided in Appendix \ref{appendix:additional_simulations}.

\subsection{Experiment Setup of MCAR}
Each arm \( a \in \{1, \dots, n\} \) has an associated mean reward \( \mu_a \), sampled independently from a uniform distribution over the interval \([0, 1]\).
The observation probability \( \gamma \)
is randomly drawn from a uniform distribution over \([0.5, 1.0]\).
At each time \( t \), when arm \( a \) is pulled, the reward \( Y_t \) is generated from a normal distribution \( \mathcal{N}(\mu_a, 1) \).
Algorithm~\ref{alg:mcar_algorithm}'s performance is reported across 20 independent runs in the MCAR environment over a time horizon of \( T = 10{,}000 \) iterations, with a fixed parameter \( \alpha = 2 \). Figure~\ref{fig:MCAR-1} depicts the cumulative regret for different \( \gamma \) values. As expected, when \( \gamma \) decreases, the regret grows more rapidly as a consequence of lower observation likelihood.


\begin{figure*}[t]
    \centering
    \begin{subfigure}[b]{0.31\textwidth}
        \includegraphics[width=\textwidth]{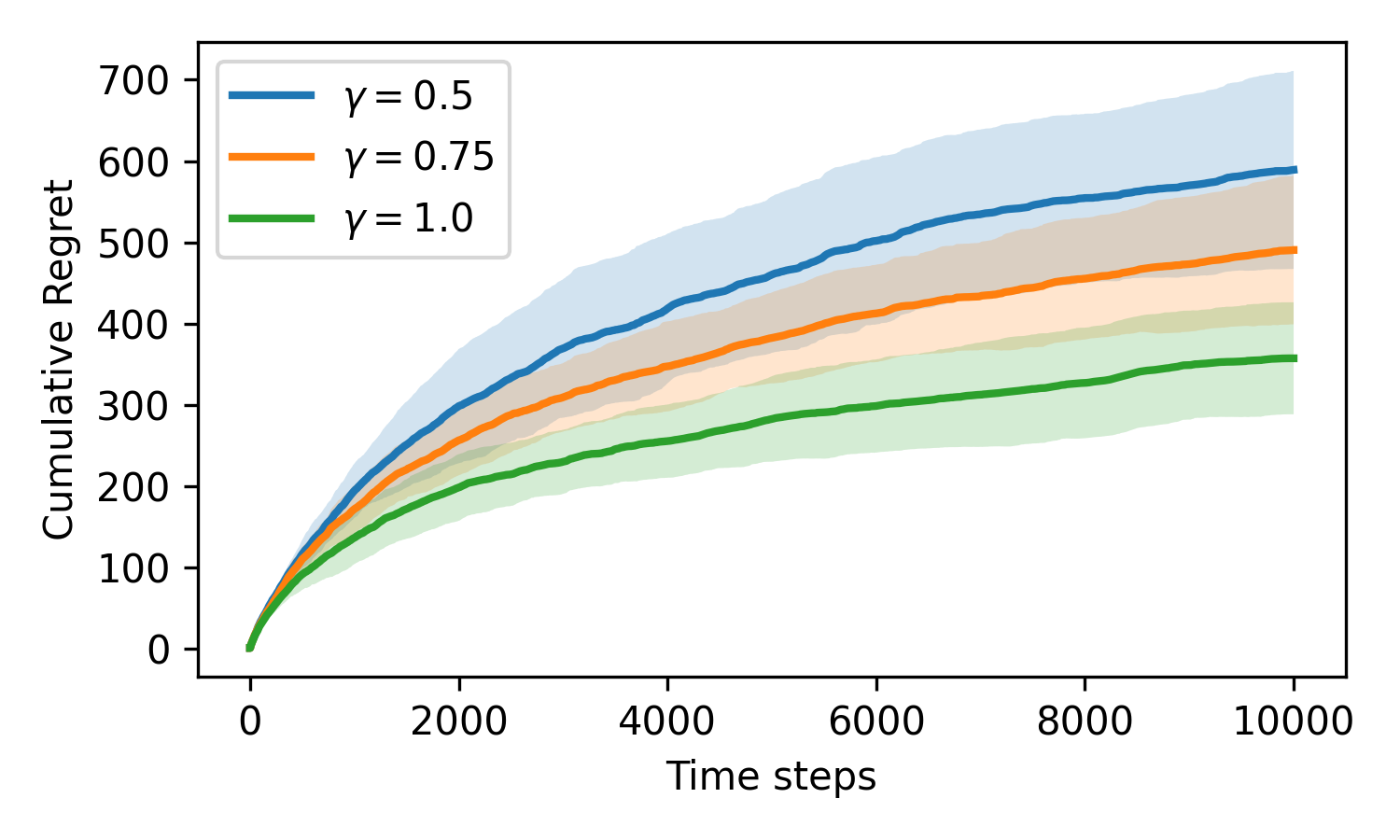}
        \caption{MCAR algorithm on MCAR bandit with various \( \gamma \) values.}
        \label{fig:MCAR-1}
    \end{subfigure}
    \hfill
    \begin{subfigure}[b]{0.31\textwidth}
        \includegraphics[width=\textwidth]{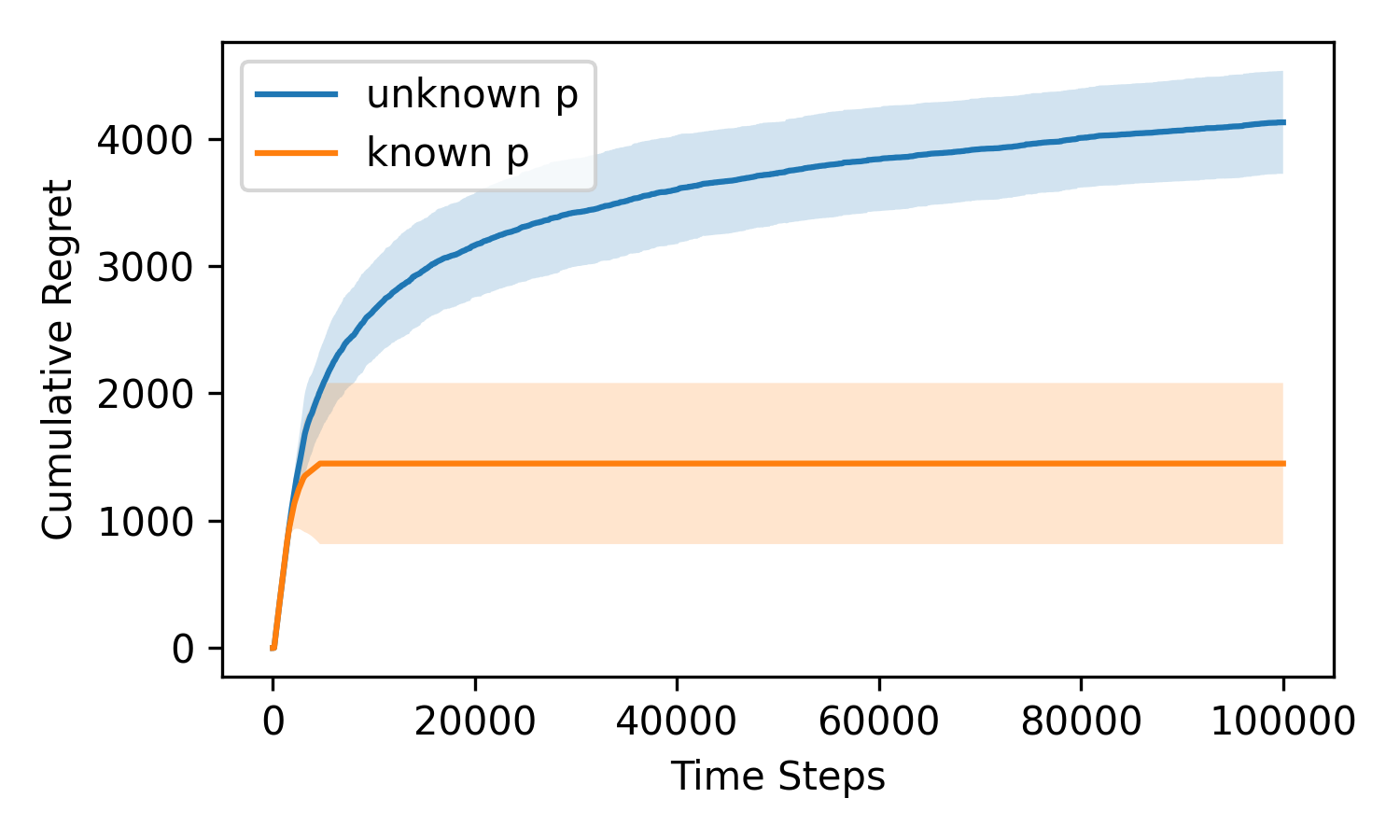}
        \caption{MAR algorithms with known and unknown p for comparison.}
        \label{fig:MAR-1}
    \end{subfigure}
    \hfill
    \begin{subfigure}[b]{0.31\textwidth}
        \includegraphics[width=\textwidth]{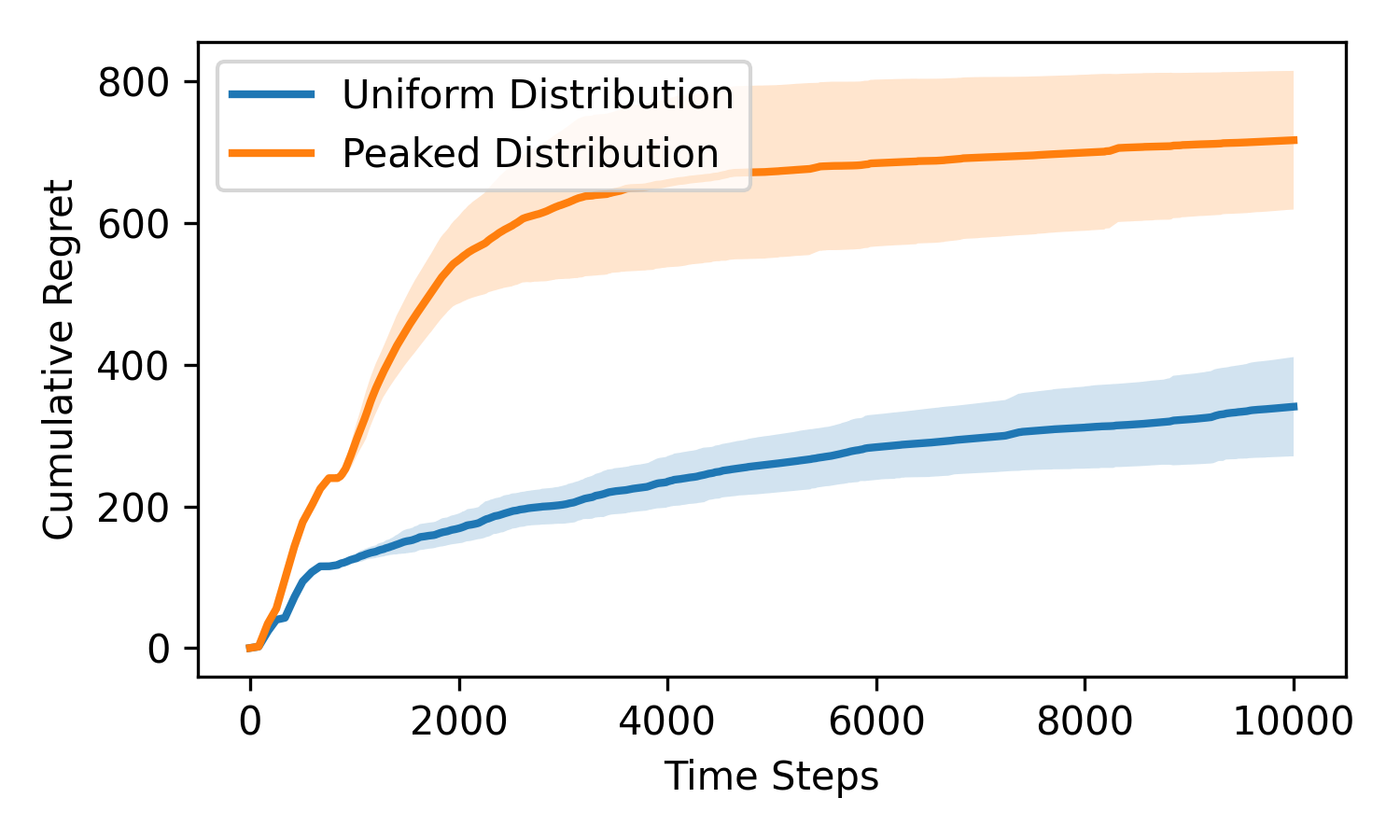}
        \caption{MAR algorithm with different p initializations on MAR environment.}
        \label{fig:MAR-2}
    \end{subfigure}
    \begin{subfigure}[b]{0.31\textwidth}
        \includegraphics[width=\textwidth]{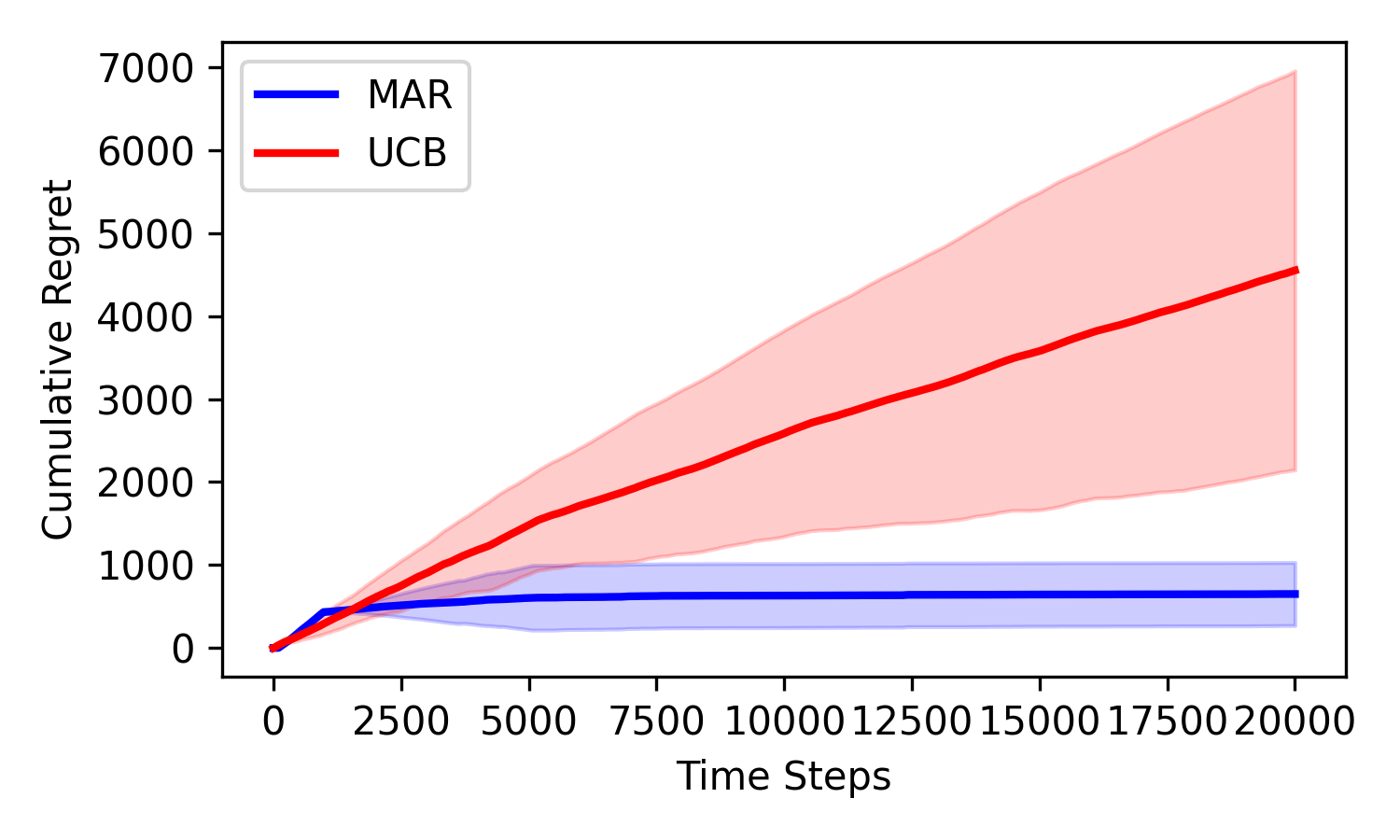}
        \caption{MAR and UCB algorithms in the MAR bandit environment.}
        \label{fig:MAR-3}
    \end{subfigure}
    \hspace{1cm}
    \begin{subfigure}[b]{0.31\textwidth}
        \includegraphics[width=\textwidth]{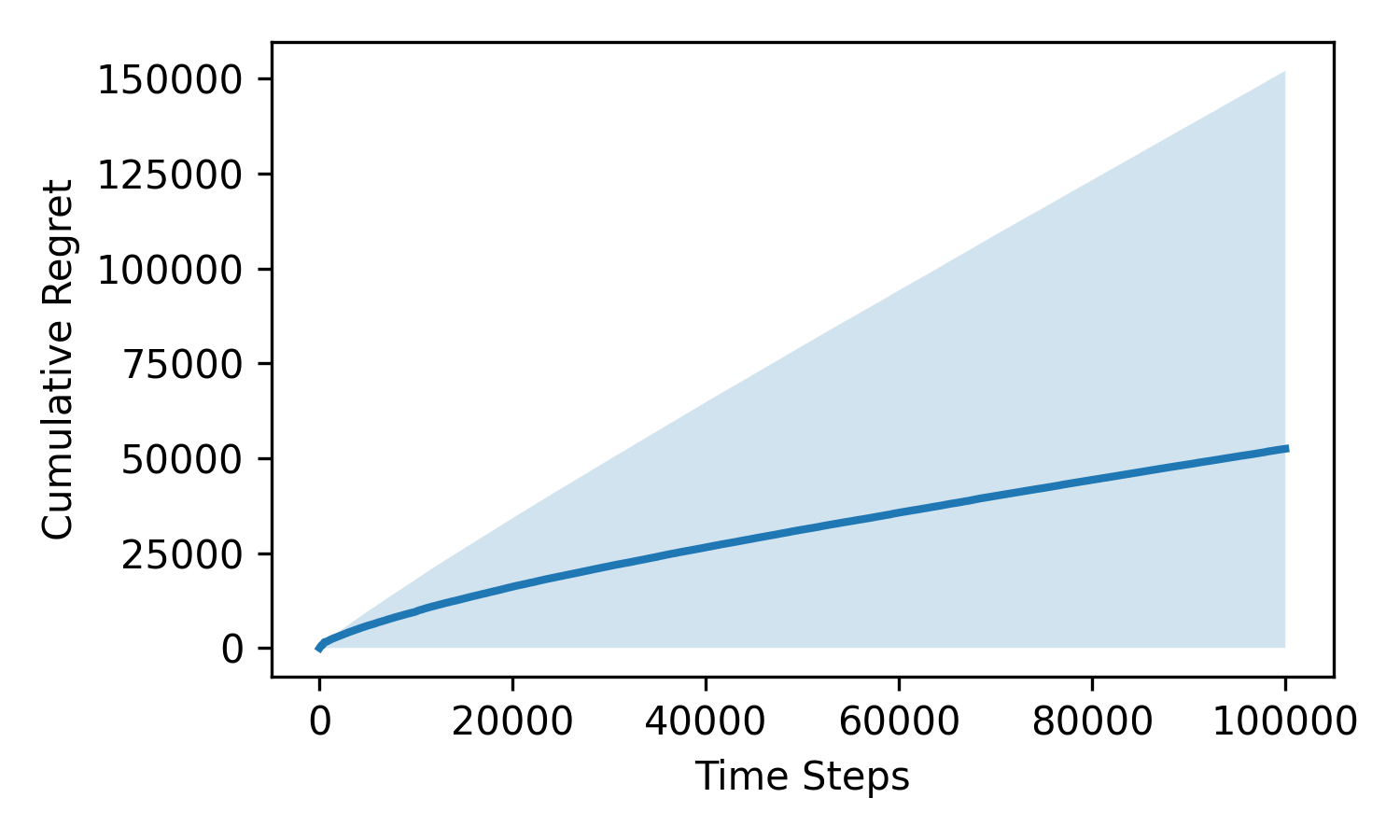}
        \caption{Performance of the MNAR algorithm in the described environment.}
        \label{fig:MNAR-1}
    \end{subfigure}
    \hfill
    \caption{Results corresponding to MCAR, MAR, and MNAR settings. The shaded regions represent the error bars, showing one standard deviation across multiple runs of the simulations.}
    \label{fig:combined}
\end{figure*}

\subsection{Experiment Setup of MAR}

The MAB environment is again modeled with \( n = 10 \) arms but with \( K = 5 \) possible mediator values. 
The expected reward for all arms is determined by \( \{\mu_{m,a}\}_{m,a} \in \mathbb{R}^{n \times K} \), where \( \mu_{m,a} \) represents the mean reward for arm \( a \) when the mediator takes value \( m\). 
The latter reward matrix is chosen by sampling each \( \mu_{m,a} \) independently from a uniform distribution over \([0, 0.4]\). To ensure the first arm is the optimal one, an additional 0.6 is added to its corresponding value.
In addition to the reward structure, the observation mechanism is defined by a matrix \(\{\gamma_{m,a}\}_{m,a} \in \mathbb{R}^{n \times K} \), where each \( \gamma_{m,a} \) is sampled independently from a uniform distribution over \([0.8, 1]\).

For each arm \( a \), a categorical probability distribution \( \{p_{m,a}\}_{m} \in \mathbb{R}^{K} \) is defined over the \( K \) values of $M$.
This distribution is drawn from a Dirichlet distribution, i.e., \( \{p_{m,a}\}_{m} \sim \text{Dirichlet}(\mathbf{1}_K) \). 
Upon pulling arm $a$ and the mediator taking value $m$,  reward \( Y_t \) is drawn from a normal distribution \( \mathcal{N}(\mu_{m,a}, 1) \), where \( \mu_{m,a} \) is the mean reward for arm \( a \) when mediator takes value \( m \). 
The reward is observed with probability \( \gamma_{m,a} \).
We ran Algorithms~\ref{alg:mar_algorithm2} and~\ref{alg:mnar_algorithm} over a time horizon of \( T = 100,000 \). Their cumulative regret was averaged across 10 independent runs.  
As shown in Fig.~\ref{fig:MAR-1}., knowing conditional probabilities $p_{m,a}$ in advance improves the cumulative regret, as expected.
%

Fig.~\ref{fig:MAR-2} demonstrates the average cumulative regret of the MAR algorithm with different probability distributions over the mediator.
In particular, 
two mediator value selection strategies were tested:
(i) uniform, where each mediator value has an equal probability, and (ii) a peaked distribution, where one mediator per arm has a higher probability, using a Dirichlet distribution biased by \( \alpha = 5 \) for the chosen mediator.
The peaked distribution results in a higher cumulative regret, which aligns with the result from Theorem~\ref{theo:mar_lower_1}, since \( S \) is maximized when the probability distribution is concentrated on the largest \( \gamma_{m, a} \).


In Figure~\ref{fig:MAR-3}, we compare the performance of the UCB and MAR algorithms in the MAR bandit environment. The results illustrate that the cumulative regret of the UCB algorithm is consistently higher than that of the MAR algorithm. Notably, the regret of the UCB algorithm exhibits a near-linear growth, as a result of the bias in its estimation of the reward. This bias is due to the failure to account for the mediator structure. In contrast, the MAR algorithm, which explicitly utilizes mediators to handle missingness, achieves accurate reward estimation and a significantly lower regret.


\subsection{Experiment Setup of MNAR}
The MNAR algorithm was evaluated in an environment with \( n = 10 \) arms, \( K = 5 \) mediators, and \( \vert\mathcal{Y}\vert = 5 \) possible outcomes, over a horizon of \( T = 100{,}000 \), repeated 10 times. For each arm \( a \) and mediator \( m \), the reward function followed a categorical distribution sampled from a Dirichlet distribution, except for one arm which was sampled from a biased Dirichlet distribution. The bias was applied to the largest \( y \in \mathcal{Y} \), ensuring that this arm had a higher expected reward. 
The observation probabilities \( \gamma_{y,a} \) were drawn from a uniform distribution over \([0.5, 1.0]\), while the mediator probabilities were sampled from a Dirichlet distribution.
Fig.~\ref{fig:MNAR-1} shows that the algorithm successfully adapts to the MNAR setup, effectively minimizing the cumulative regret. 

\section{Limitations and Concluding Remarks}\label{sec:conc}
We studied multi-armed bandits with missing outcomes and adapted UCB algorithms to incorporate missingness.
Our approaches extend the applicability of MAB algorithms to a wider range of real-world online decision-making problems.
We expect that the insights given by this paper will help researchers to develop and adapt other existing decision-making algorithms to take missingness into account.
We assumed that the auxiliary (mediator) $M$ takes values in a finite set.
Parametric (or semiparametric) models can be adopted to relax this assumption in the future.
We further acknowledge that estimating the odds ratios through integral equations in the MNAR setting presents significant challenges, both in terms of computational complexity and sample efficiency. Hence, we have postponed the problem of MAB with continuous outcomes missing not at random to future work.
\bibliographystyle{plainnat}
\bibliography{bibliography}

\clearpage
\appendix
\onecolumn
\begin{center}
    \Large\bfseries
    \vspace{-1em}Appendix\vspace{-.1em}
\end{center}
This appendix is organized as follows.
Section \ref{apx:alg} includes the omitted algorithms referred to in the main text.
Section \ref{apx:proofs} includes the technical proofs of our results. 
Section \ref{apx:Missing M} provides our algorithm for the case where both the outcome and the mediator are missing not at random (MNAR) along with the regret analysis and proofs.
Finally, Section \ref{appendix:additional_simulations} includes additional empirical evaluation results.
\vspace{-.5em}
\section{Main Algorithms}\label{apx:alg}
\begin{algorithm}[H]
\caption{MCAR algorithm}\label{alg:mcar_algorithm}
\begin{algorithmic}[1]
\State \textbf{Input:} Number of arms $n$, time horizon $T$, $\alpha \geq 1$
\State Initialize: $\hat{\mu}_a = 0$ for all arms $a = 1, 2, \dots, n$ \Comment{initial mean reward estimate for each arm}
\State Set: $T_{a, o} = 0$ for all arms $a = 1, 2, \dots, n$ \Comment{number of times each arm is pulled and reward is observed}

\For{each round $t = 1, 2, \dots, T$}
    \For{each arm $a = 1, 2, \dots, n$}
        \State $\text{UCB}_a(t) = \hat{\mu}_a + \sqrt{\frac{\alpha \log(T)}{2 T_{a, o}}}$

    \EndFor
    \State Select arm $a_t = \arg \max_a \text{UCB}_a(t)$
    \State Pull arm $a_t$ and observe reward $r_t$
    \If{reward is observed}
        \State Update $n_{a_t}$ and $\hat{\mu}_{a_t}$
    \EndIf
\EndFor
\end{algorithmic}
\end{algorithm}
\begin{algorithm}[H]
\caption{MAR Algorithm with known \( p_{m,a} \)}\label{alg:mar_algorithm}
\begin{algorithmic}[1]
\State \textbf{Input:} Number of arms $n$, time horizon $T$, exploration parameter $\alpha$
\State \textbf{Initialize:}  
\For{each arm $a \in [n]$ and $m \in \mathbb{M}$} 
\State $\hat{\mu}_{m, a} = 0$ \Comment{estimated mean reward for arm $a$ when $M = m$} 
\State $T_{m, a, o} = 0$ \Comment{number of times arm $a$ is pulled with $M = m$ and reward observed} 
\State $s_{m, a} = 0$ \Comment{number of times $M = m$ was observed for arm $a$}
\EndFor

\For{each arm $a \in [n]$} \For{$\log(T)^2$ rounds}
    \State Pull arm $a$, observe $m$ and reward $r$
    \State Update $s_{m, a}$ for observed $M = m$
    \If{reward is observed} \State Update $T_{m, a, o}$ and $\hat{\mu}_{m, a}$ \EndIf
\EndFor \EndFor

\State $T_1 = n \log(T)^2 $, $T_2 = T - T_1$

\For{each round $t = 1, \dots, T_2$}
    \For{each arm $a \in [n]$}
        \State Compute $\hat{\mu}_{a} = \sum\limits_{m \in [K]} p_{m, a} \hat{\mu}_{m, a}$ \Comment{estimated mean reward for arm $a$}
        \State Compute $\text{UCB}_a(t) = \hat{\mu}_a + \sqrt{\frac{\alpha \log(T)}{2} \sum\limits_{m \in \mathbb{M}} \frac{p_{m, a}^2}{T_{m, a, o}}}$ \Comment{Upper Confidence Bound for arm $a$}
    \EndFor
    \State Select arm $a_t = \arg \max_a \text{UCB}_a(t)$
    \State Pull arm $a_t$, observe $m$ and reward $r_t$
    \State Update $s_{m,a_t}$ and, if reward is observed, update $n_{a_t, m}$ and $\hat{\mu}_{a_t, m}$
\EndFor

\end{algorithmic}
\end{algorithm}

\clearpage
\begin{algorithm}[H]
\caption{MAR Algorithm with unknown \( p_{m,a} \)}\label{alg:mar_algorithm2}
\begin{algorithmic}[1]
\State \textbf{Input:} Number of arms $n$, time horizon $T$, exploration parameter $\alpha$
\State \textbf{Initialize:} 
\For{each arm $a \in [n]$ and $m \in \mathbb{M}$} 
\State $\hat{\mu}_{m, a} = 0$ \Comment{estimated mean reward for arm $a$ when $M = m$} 
\State $T_{m, a, o} = 0$ \Comment{number of times arm $a$ is pulled with $M = m$ and reward observed} 
\State $s_{m, a} = 0$ \Comment{number of times $M = m$ was observed for arm $a$}
\EndFor

\For{each arm $a \in [n]$} \For{$\log(T)^2$ rounds}
    \State Pull arm $a$, observe $m$ and reward $r$
    \State Update $s_{m, a}$ for observed $M = m$
    \If{reward is observed} \State Update $T_{m, a, o}$ and $\hat{\mu}_{m, a}$ \EndIf
\EndFor \EndFor

\State $T_1 = n \log(T)^2$, $T_2 = T - T_1$

\For{each round $t = 1, \dots, T_2$}
    \For{each arm $a \in [n]$}
        \State Estimate \( \hat{p}_{m,a} = \frac{s_{m,a}}{T_a} \) for each $m \in \mathbb{M}$
        \State Compute $\hat{\mu}_{a} = \frac{1}{T_a} \sum_t\sum_{m\in\mathcal{M}}\hat{\mu}_{m,a}  \mathbbm{1}\{M_t=m\}$
        \State Compute $\text{UCB}_a(t) = \hat{\mu}_a + 8\sqrt{\frac{\alpha \log(T)}{2} \sum\limits_{m \in \mathbb{M}} \frac{\hat{p}_{m, a}^2}{T_{m, a, o}}}$
    \EndFor
    \State Select arm $a_t = \arg \max_a \text{UCB}_a(t)$
    \State Pull arm $a_t$, observe $m$ and reward $r_t$
    \State Update $s_{m,a_t}$ and, if reward is observed, update $n_{m, a_t}$ and $\hat{\mu}_{m, a_t}$
\EndFor

\end{algorithmic}
\end{algorithm}

\begin{algorithm}[H]
\caption{MNAR Algorithm}\label{alg:mnar_algorithm}
\begin{algorithmic}[1]
\State \textbf{Input:} Number of arms $n$, time horizon $T$, exploration parameter $\alpha$
\State \textbf{Initialize:} 
\For{each arm $a \in [n]$ and $m \in \mathbb{M}$} 
    \State Set $b_{m, 0 | a} = 0$ \Comment{Estimation of \( \mathbb{P}(M = m, O^Y = 0 \mid a) \)}
    \State Set $\hat{\theta}_{a} = [0]_{k \times L}$ \Comment{Estimation of matrix \( \theta_a[m, y] = \mathbb{P}(m, y, O^Y = 1 \mid a) \)}
    \State Set $q_{m, y \mid 1, a} = 0$ \Comment{Estimation of \( \mathbb{P}(M = m, Y = y \mid a, O^Y = 1) \)}
    \State Set $T_{a} = 0$ \Comment{Count of pulls of arm $a$}
    \State Set $T_{a, o} = 0$ \Comment{Count of pulls of arm $a$ with observed reward}
\EndFor

\For{each arm $a \in [n]$} 
    \For{$\log(T)^2$ rounds}
        \State Pull arm $a$, observe mediator $m$ and reward $y$
        \State Update $T_a$ 
        \If{reward is observed} 
            \State Update $T_{a, o}$, $\hat{\theta}_{a}[m, y]$, and $q_{m, y \mid 1, a}$ 
        \Else
            \State Update $b_{m, 0 | a}$
        \EndIf
    \EndFor
\EndFor

\State Set $T_1 = n \log(T)^2$ and $T_2 = T - T_1$

\For{each round $t = 1, \dots, T_2$}
    \For{each arm $a \in [n]$}
        \State Solve $x_a = \hat{\theta_a}^{\dagger} b_a$ 
        \State Update $x_a = x_a + [1]_{L \times 1}$
        \State Compute $\hat{p}(m, y) = x_a[y] \times q_{m, y \mid 1, a}$ 
        \State Compute $\hat{p}(y) = \sum\limits_{m \in \mathbb{M}} \hat{p}(m, y)$ 
        \State Compute $\hat{\mu}_a = \sum\limits_{y \in \mathbb{Y}} y \times \hat{p}(y)$ 
        \State Compute $\hat{\gamma}_a = \frac{1}{\max\limits_{y \in \mathbb{Y}}(x_a[y])}$ 
        \State Compute $\text{UCB}_a = \hat{\mu}_a + 8\frac{L C_a}{\lVert \hat{\theta}_a \rVert_\infty \hat{\gamma}_a}\sqrt{\frac{\alpha \log(T)}{T_a}} + \frac{K}{\hat{\gamma}_a}\sqrt{\frac{\alpha \log(T)}{T_{a, o}}}$ 
    \EndFor
    \State Select arm $a_t = \arg \max_a \text{UCB}_a(t)$ 
    \State Pull arm $a_t$, observe $m$ and reward $y_t$
    \State Update $T_{a_t}$ 
    \If{reward is observed} 
        \State Update $n_{a_t}$, $\hat{\theta}_{{a_t}}[m, y]$, and $q_{m, y \mid 1, {a_t}}$ 
    \Else
        \State Update $b_{m, 0 | {a_t}}$
    \EndIf
\EndFor

\end{algorithmic}
\end{algorithm}

\clearpage
\section{Technical Proofs}\label{apx:proofs}
\paragraph{Double Robustness of AIPW estimator.}
Following \nameref{discussion2} in Section \ref{subsec:mar}, let $\hat{\gamma}_{m,a}$ and $\hat{\mu}_{m,a}$ be models for $\gamma_{m,a}$ and $\ex{Y^o_t\mid m,a,\oo_t=1}$, respectively.
Define
\begin{equation}\label{eq:dre}
\begin{split}
    \hat{\mu}_a &= \mathbb{E}\big[
        \sum_{m\in\mathcal{M}}\frac{\mathbbm{1}\{M_t=m\}}{\hat{\gamma}_{m,a}}\big(
        Y^o_t\mathbbm{1}\{\oo_t=1\}
        -
        (\mathbbm{1}\{\oo_t=1\}-\hat{\gamma}_{m,a})\hat{\mu}_{m,a}
        \big)
        \mid A_t=a
    \big],
\end{split}
\end{equation}
as an estimator for $\mu_a$ of Eq.~\eqref{eq:dr}.
Herein, we prove that $\hat{\mu}_a$ is \emph{doubly robust}, in the sense that if either of the missingness probability models ($\hat{\gamma}_{m,a}$) or the outcome regression models ($\hat{\mu}_{m,a}$), but not necessarily both, are correctly specified, then $\hat{\mu}_a$ of Eq.~\eqref{eq:dre} is consistent for $\mu_a$ of Eq.~\eqref{eq:dr}.
We discuss the two cases separately:

Case (i): the missingness probabilities are correctly specified; i.e., $\hat{\gamma}_{m,a}=\gamma_{m,a}$.
In this case,
\[
\begin{split}
&\ex{\sum_{m\in\mathcal{M}}(\frac{\mathbbm{1}\{\oo_t=1\}}{\hat{\gamma}_{m,a}}-1)\hat{\mu}_{m,a}\mathbbm{1}\{M_t=m\}\mid A_t=a}
\\&\overset{(a)}{=}
\sum_{m\in\mathcal{M}}\ex{(\frac{\mathbbm{1}\{\oo_t=1\}}{{\gamma}_{m,a}}-1)\hat{\mu}_{m,a}\mathbbm{1}\{M_t=m\}\mid A_t=a}
\\&\overset{(b)}{=}
\sum_{m\in\mathcal{M}}\ex{(\frac{\mathbbm{1}\{\oo_t=1\}}{{\gamma}_{m,a}}-1)\mid A_t=a,M_t=m}
\hat{\mu}_{m,a}p_{m,a}
\\&\overset{(c)}{=}
\sum_{m\in\mathcal{M}}(\frac{\gamma_{m,a}}{{\gamma}_{m,a}}-1)
\hat{\mu}_{m,a}p_{m,a}\\&=0,
\end{split}
\]
where $(a)$ is due to $\hat{\gamma}_{m,a}$ being correctly specified,
$(b)$ is an application of the law of total expectation, and $(c)$ is by definition of $\gamma_{m,a}=\ex{\oo_t\mid A_t=a,M_t=m}$.
As a result, we get
\[
\hat{\mu}_a= \mathbb{E}\big[
        \sum_{m\in\mathcal{M}}\frac{\mathbbm{1}\{M_t=m\}}{{\gamma}_{m,a}}
        Y^o_t\mathbbm{1}\{\oo_t=1\}\mid A_t=a
    \big],
\]
which matches Eq.~\eqref{eq:ht}, and therefore $\hat{\mu}_a$ is consistent for $\mu_a$.

Case (ii): the outcome regression models are correctly specified; i.e., 
$\hat{\mu}_{m,a}=\ex{Y^o_t\mid m,a,\oo_t=1}$.
Then,
\[\begin{split}
    &\mathbb{E}\big[
        \sum_{m\in\mathcal{M}}\frac{\mathbbm{1}\{M_t=m\}}{\hat{\gamma}_{m,a}}(
        Y^o_t
        -
        \hat{\mu}_{m,a})\mathbbm{1}\{\oo_t=1\}
        \mid A_t=a
    \big]
    \\&\overset{(a)}{=}
    \sum_{m\in\mathcal{M}}\mathbb{E}\big[
        \frac{\mathbbm{1}\{\oo_t=1\}}{\hat{\gamma}_{m,a}}(
        Y^o_t
        -
        \hat{\mu}_{m,a})
        \mid A_t=a, M_t=m
    \big]p_{m,a}
    \\&\overset{(b)}{=}
    \sum_{m\in\mathcal{M}}\frac{\gamma_{m,a}}{\hat{\gamma}_{m,a}}\mathbb{E}\big[
        Y^o_t
        -
        \hat{\mu}_{m,a}
        \mid A_t=a, M_t=m, \oo_t=1
    \big]p_{m,a}
    \\&\overset{(c)}{=}
    \sum_{m\in\mathcal{M}}\frac{\gamma_{m,a}}{\hat{\gamma}_{m,a}}\big(\mathbb{E}\big[
        Y^o_t
        \mid A_t=a, M_t=m, \oo_t=1
    \big]-
        \hat{\mu}_{m,a}\big)p_{m,a}
        \\&\overset{(d)}{=}0,
    \end{split}
\]
where $(a)$ and $(b)$ are due to the law of total expectations, $(c)$ is by linearity of expectation, and $(d)$ follows from the correctness of $\hat{\mu}_{m,a}$.
From Eq.~\eqref{eq:dre},
\[
\begin{split}
    \hat{\mu}_{m,a}&=\ex{\sum_{m\in\mathcal{M}}\mathbbm{1}\{M_t=m\}\hat{\mu}_{m,a}\mid A_t=a}\\
    &= \hat{\mu}_{m,a}=\ex{\sum_{m\in\mathcal{M}}\mathbbm{1}\{M_t=m\}\ex{Y^o_t\mid m, a, \oo_t=1}\mid A_t=a}
    \\ &= \hat{\mu}_{m,a}=\ex{\ex{Y^o_t\mid M, a, \oo_t=1}\mid A_t=a},
\end{split}
\]
which matches Eq.~\eqref{eq:idmar}, and therefore $\hat{\mu}_{m,a}$ is consistent for $\mu_{m,a}$.\qed
\theomcarupper*
\begin{proof}
Let \( a^* = \arg\max\limits_{a} \mu_a \) be the optimal arm. Using Hoeffding's inequality, we can derive the following bounds for any time step \( 1 \leq t \leq T \):

- If \( a = a_t = \arg\max\limits_{a} \left( \text{UCB}_a \right) \), we have:
\[
    \hat{\mu}_{a} - \mu_{a} \leq \sqrt{\frac{\alpha \log(t)}{2T_{a, o}}},
\]
with probability \( 1 - t^{-\alpha} \). Name this ``good event" \( A_{t} \).

- If \( a = a^* \), we have:
\[
    \mu_{a} - \hat{\mu}_{a} \leq \sqrt{\frac{\alpha \log(t)}{2T_{a, o}}},
\]
with probability \( 1 - t^{-\alpha} \).  Name this ``good event" \( B_{t} \).

Now, define \( \epsilon_a = \sqrt{\frac{\alpha \log(t)}{2T_{a, o}}} \). For \( a = a_t = \arg\max\limits_{a} \left( \text{UCB}_a \right) \), we get the following inequality:
\begin{equation}
    \mu_a + 2\epsilon_a \geq \hat{\mu}_a + \epsilon_a = \text{UCB}_a \geq \text{UCB}_{a^*} = \hat{\mu}_{a^*} + \epsilon_{a^*} \geq \mu_{a^*} \quad \Rightarrow \quad \epsilon_a \geq \frac{\Delta_a}{2},
    \label{eq:mcar_proof_1}
\end{equation}
where \( \Delta_a = \mu_{a^*} - \mu_a \).

Now, if \( E_t \) represents the ``good events" at time step \( t \), then under \( E = \bigcap\limits_{t} E_t \), using \eqref{eq:mcar_proof_1} we obtain:
\[
T_{a, o} \leq 4\alpha \log(T) \Delta_a^{-2}.
\]

Thus, we have:
\begin{align}
    \mathbb{E}[T_{a, o}] &= \sum\limits_{t=1}^T \mathbb{E}[\mathbb{I}(I_t = a, O^Y_t = 1)] \nonumber \\
    &\leq 4\alpha \log(T) \Delta_a^{-2} + \sum\limits_{t=1}^T \mathbb{E}[\mathbb{I}(E_t^c)]
    \nonumber \\
    &= 4\alpha \log(T) \Delta_a^{-2} + \sum\limits_{t=1}^T \mathbb{E}[\mathbb{I}(\left( A_t^c \cup B_t^c  \right)   )] \nonumber \\
    &\leq 4\alpha \log(T) \Delta_a^{-2} + \sum\limits_{t=1}^T 2t^{-\alpha} \nonumber \\
    &\leq 4\alpha \log(T) \Delta_a^{-2} + \frac{2\alpha}{\alpha - 1}.
\end{align}

Since we observe the reward with probability \( \gamma \), and \( O^Y \indep (A, Y) \), we have \( \mathbb{E}[T_{a, o}] = \gamma \mathbb{E}[T_a] \). Therefore:
\[
\mathbb{E}[T_a] \leq \frac{4\alpha \log(T) \Delta_a^{-2} + \frac{2\alpha}{\alpha - 1}}{\gamma}.
\]

Let \( x = \sqrt{\frac{4\alpha n \log(T)}{T\gamma}} \). Then, we have:
\begin{align}
    \mathbb{E}[R_T] &= \sum\limits_{a} \Delta_a \mathbb{E}[T_a] \nonumber \\
                  &= \sum\limits_{\Delta_a < x} \Delta_a \mathbb{E}[T_a] + \sum\limits_{\Delta_a \geq x} \Delta_a \mathbb{E}[T_a] \nonumber \\
                  &\leq Tx + \sum\limits_{\Delta_a \geq x} \Delta_a \frac{4\alpha \log(T) \Delta_a^{-2} + \frac{2\alpha}{\alpha - 1}}{\gamma} \nonumber \\
                  &= Tx + \frac{4n\alpha \log(T)}{x\gamma} + \frac{2n\alpha}{(\alpha - 1)\gamma} \nonumber \\
                  &= 2\sqrt{\frac{4n\alpha T \log(T)}{\gamma}} + \frac{2n\alpha}{\gamma(\alpha - 1)} \\
                  &= O\left(\sqrt{\frac{n\alpha T \text{log}(T)}{\gamma}}\right)
\end{align}

\end{proof}

\theomcarlower*
\begin{proof}
Consider the following $n + 1$ bandit instances, with $n$ arms labeled $a_1, a_2, \dots, a_n$.

\textbf{Bandit instance $0$:}
\begin{itemize}
    \item $\mathbb{E}[Y(a)] = 0$ for all $a = a_1, \dots, a_n$.
\end{itemize}

\textbf{Bandit instance $k$ for $k = 1, \dots, n$:}
\begin{itemize}
    \item $\mathbb{E}[Y(a_k)] = \Delta$ for $a = a_k$.
    \item $\mathbb{E}[Y(a)] = 0$ for $a \neq a_k$.
\end{itemize}

Next, we present key lemmas adapted from \cite{lattimore2020bandit} to complete our analysis.

\textbf{Divergence Decomposition:}  
Let $\nu = (P(1), \dots, P(k))$ and $\nu' = (P'(1), \dots, P'(k))$ represent the reward distributions for two $k$-armed bandits. For a fixed policy $\pi$, let $P_\nu = P_{\nu,\pi}$ and $P_{\nu'} = P_{\nu',\pi}$ be the probability measures induced by the $n$-round interaction with $\nu$ and $\nu'$. Then:
\[
\text{KL}(P_\nu, P_{\nu'}) = \sum_{i=1}^{k} \mathbb{E}_\nu[T_i(n)] \text{KL}(P(i), P'(i)).
\]

\textbf{Pinsker's Inequality:}  
For measures $P$ and $Q$ on the same probability space $(\Omega, \mathcal{F})$, the total variation distance is bounded by:
\[
d_\text{TV}(P, Q) = \sup_{A \in \mathcal{F}} |P(A) - Q(A)| \leq \sqrt{\frac{1}{2} \text{KL}(P, Q)}.
\]

\textbf{Total Variation Bound:}  
Let $(\Omega, \mathcal{F})$ be a measurable space, and let $P$ and $Q$ be probability measures on $\mathcal{F}$. For any $\mathcal{F}$-measurable random variable $X : \Omega \to [a, b]$, we have:
\[
\left| \int_\Omega X(\omega) dP(\omega) - \int_\Omega X(\omega) dQ(\omega) \right| \leq (b - a) d_\text{TV}(P, Q).
\]

Now, in our setup with missing observations, so $(O^Y, Y)$ represent the observation tuple. Hence, we have:
\[
\text{KL}(P_0, P_i) = \mathbb{E}_0 [T_i] \text{KL}(P_0(i), P_i(i)) = \mathbb{E}_0 [T_i] \frac{\gamma \Delta^2}{2}.
\]

From this, we can bound $\mathbb{E}_i [T_i(T)]$ as follows:
\begin{align*}
\mathbb{E}_i [T_i(T)] &\leq \mathbb{E}_0 [T_i(T)] + T d_\text{TV}(P_0(i), P_i(i)) \\
&\leq \mathbb{E}_0 [T_i(T)] + T \sqrt{\frac{1}{2} \text{KL}(P_0(i), P_i(i))} \\
&= \mathbb{E}_0 [T_i(T)] + T \sqrt{\frac{1}{2} \cdot \frac{\gamma \Delta^2}{2} \mathbb{E}_0 [T_i(T)]} \\
&= \mathbb{E}_0 [T_i(T)] + \frac{T}{2} \sqrt{\gamma \Delta^2 \mathbb{E}_0 [T_i(T)]}.
\end{align*}

Let $R_i = R_T(\pi; i)$ denote the regret of applying policy $\pi$ on the $i$-th bandit instance up to time $T$, where $i$ refers to the $i$-th bandit instance.

Summing over all bandit instances, we have:
\begin{align*}
\sum_{i=1}^n \mathbb{E}[R_i] 
&= \sum_{i=1}^n \Delta(T - \mathbb{E}_i[T_i(T)]) \\
&\geq \Delta Tn - \Delta \sum_{i=1}^n \left(\mathbb{E}_0 [T_i(T)] + \frac{T}{2} \sqrt{\gamma \Delta^2 \mathbb{E}_0 [T_i(T)]}\right) \\
&\geq \Delta Tn - \Delta T - \frac{\Delta^2 T}{2} \sqrt{\gamma Tn} \\
&\geq \frac{\Delta Tn}{2} - \frac{\Delta^2 T}{2} \sqrt{\gamma Tn} \quad \text{using } \Delta = \frac{n}{2\sqrt{\gamma Tn}} \\
&\geq \frac{Tn^2}{8\sqrt{\gamma Tn}} = \frac{n}{8} \sqrt{\frac{Tn}{\gamma}}.
\end{align*}

Thus, there exists an instance where $\mathbb{E}[R_i] \geq \Omega\left(\sqrt{\frac{Tn}{\gamma}}\right)$.
\end{proof}

\theomarupperfirst*
\begin{proof}
As before, let \( a^* = \arg\max\limits_{a} \mu_a \) denote the optimal arm, and define \( T_1 = \sum\limits_{a} T_{1,a} \) as the total number of times the agent samples each arm during the initial rounds. After the first \( T_1 \) rounds, we can derive the following bounds at any time step \( 1 \leq t \leq T_2 \).

For each arm \( a \), let the reward samples observed when \( M = i \) be denoted by \( Y_{m, a}(1), \dots, Y_{m, a}(n_{m, a}) \). Applying Hoeffding's inequality, we obtain:
\[
\mathbb{P}\left( \left| \sum_{m \in [K]} p_{m,a} \frac{\sum_{j = 1}^{n_{m, a}} Y_{m, a}(j)}{n_{m, a}} - \mu_a \right| \geq \epsilon_a \right) \leq 2\exp\left( -\frac{2\epsilon_a^2}{\sum_{m \in [K]} \frac{p_{m,a}^2}{n_{m, a}}} \right)
\]

This result holds because the sub-Gaussian norm of the random variable \( p_{m,a} \frac{Y_{m, a}(j)}{n_{m, a}} \) is \( \frac{p_{m,a}}{n_{m, a}} \).

By setting \( \epsilon_a = \sqrt{\frac{\alpha \log(t)}{2} \sum_{m \in [K]} \frac{p_{m,a}^2}{n_{m, a}}} \), we obtain the following inequality, which holds with probability at least \( 1 - 2t^{-\alpha} \):
\[
|\hat{\mu}_a - \mu_a| \leq \sqrt{\frac{\alpha \log(t)}{2} \sum_{m \in [K]} \frac{p_{m,a}^2}{n_{m, a}}}
\]
Name the above ``good event" \( A_{t, a} \).

Also, like before for \( a = a_t = \arg\max\limits_{a} \left( \text{UCB}_a \right) \), we get the following inequality:
\begin{equation}
    \mu_a + 2\epsilon_a \geq \hat{\mu}_a + \epsilon_a = \text{UCB}_a \geq \text{UCB}_{a^*} = \hat{\mu}_{a^*} + \epsilon_{a^*} \geq \mu_{a^*} \quad \Rightarrow \quad \epsilon_a \geq \frac{\Delta_a}{2},
    \label{eq:mar_proof_1}
\end{equation}
where \( \Delta_a = \mu_{a^*} - \mu_a \).

Next, let \( s_{m,a} \) represent the number of times arm \( a \) is pulled and \( M = i \) is observed, and let \( T_a \) represent the total number of times arm \( a \) is pulled. Using Hoeffding's inequality, we can bound the deviation between \( p_{m,a} \) (the probability of observing \( M = i \)) and the empirical ratio \( \frac{s_{m,a}}{T_a} \) as follows:
\[
    p_{m,a} - \frac{s_{m,a}}{T_a} \leq \sqrt{\frac{\alpha \log(t)}{2T_a}} \leq \sqrt{\frac{\alpha \log(T)}{2T_a}},
\]
with probability at least \( 1 - t^{-\alpha} \). Name this ``good event" \( B_{t, m, a} \).

Similarly, we bound the deviation between \( \gamma_{m, a} \) and \( \frac{n_{m, a}}{s_{m,a}} \), where \( n_{m, a} \) is the number of times reward is observed for arm \( a \) and \( M = i \):
\[
    \gamma_{m, a} - \frac{n_{m, a}}{s_{m,a}} \leq \sqrt{\frac{\alpha \log(t)}{2s_{m,a}}} \leq \sqrt{\frac{\alpha \log(T)}{2s_{m,a}}},
\]
again with probability \( 1 - t^{-\alpha} \).  Name this ``good event" \( C_{t, m, a} \).

For sufficiently large \( T \), we have:
\[
\log(T)^2 \geq 2\alpha \log(T) \frac{1}{p^2},
\]
which implies \( T_a \geq \log(T)^2 \geq 2\alpha \log(T) \frac{1}{p^2} \). This allows us to use inequality \( \sqrt{\frac{\alpha \log(T)}{2T_a}} \leq \frac{p_{m,a}}{2} \) to derive a lower bound for \( s_{m,a} \):
\[
    s_{m,a} \geq \frac{T_a p_{m,a}}{2}.
\]

Furthermore, since \( s_{m,a} \geq \frac{T_a p_{m,a}}{2} \) and For sufficiently large \( T \), we know that \( T_a \geq T_{1,a} = \log(T)^2 \geq 2\alpha \log(T) \frac{2}{\gamma_{m, a}^2 p} \). This gives us \( \sqrt{\frac{\alpha \log(T)}{2s_{m,a}}} \leq \frac{\gamma_{m, a}}{2} \), which allows us to establish a lower bound for \( n_{m, a} \):
\[
    n_{m, a} \geq \frac{\gamma_{m, a}}{2} s_{m,a}.
\]

Combining this with \( s_{m,a} \geq \frac{T_a p_{m,a}}{2} \), we derive:
\[
    n_{m, a} \geq \frac{T_a p_{m,a} \gamma_{m, a}}{4}.
\]

Let \( E_t \) represent the intersection of "good events" at time step \( t \).  Under \( E = \bigcap\limits_{t} E_t \), we obtain:

\begin{align*}
    \epsilon_a 
    &= \sqrt{\frac{\alpha \log(t)}{2} \sum_{m \in [K]} \frac{p_{m,a}^2}{n_{m, a}}} \\
    &\leq \sqrt{\frac{\alpha \log(T)}{2} \sum_{m \in [K]} \frac{4p_{m,a}^2}{T_a p_{m,a} \gamma_{m, a}}} \\
    &= \sqrt{\frac{2\alpha \log(T)}{T_a} \sum_{m \in [K]} \frac{p_{m,a}}{ \gamma_{m, a}}} \\
    &= \sqrt{\frac{2\alpha \log(T)}{T_a} P_a} \label{eq:epsilon_a}
\end{align*}

Using inequality \eqref{eq:mar_proof_1}, we have:
\[
    T_a \leq \frac{8 \alpha \log(T) P_a}{\Delta_a^2}
\]

Thus, we get:
\begin{align}
    \mathbb{E}[T_a] &= \sum\limits_{t=1}^T \mathbb{E}[\mathbb{I}(I_t = a)] \nonumber \\
    &\leq  \frac{8 \alpha \log(T) P_a}{\Delta_a^2} + \sum\limits_{t=1}^T \mathbb{E}[\mathbb{I}(E_t^c)] \nonumber \\
    & \leq  \frac{8 \alpha \log(T) P_a}{\Delta_a^2} + \sum\limits_{t=1}^T \mathbb{E}[\mathbb{I}(\bigcup\limits_{m} \left( B_{t, m, a}^c \cup C_{t, m, a}^c \right) \cup A_{t, a}  )] \nonumber \\
    &\leq  \frac{8 \alpha \log(T) P_a}{\Delta_a^2} + \sum\limits_{t=1}^T 4K t^{-\alpha} \nonumber \\
    &\leq  \frac{8 \alpha \log(T) P_a}{\Delta_a^2} + \frac{4K\alpha}{\alpha - 1}.
    \label{eq:mar_proof_2}
\end{align}

To conclude, note that the regret of second part of algorithm is \( \mathbb{E}[R_2] = \sum\limits_{a} \Delta_a \mathbb{E}[T_a] \). 
We now split the arms into two groups: \( \Delta_a \leq \sqrt{\frac{8\alpha \log(T)S}{T}} \) and \( \Delta_a \geq \sqrt{\frac{8\alpha \log(T)S}{T}} \). Let \( R_2 \) be the regret for the second part, and let \( x = \sqrt{\frac{8\alpha \log(T)S}{T}} \).
To conclude, note that \( \mathbb{E}[R] = \sum\limits_{a} \Delta_a \mathbb{E}[T_a] \). 
Since \(S = \sum\limits_{a} P_a \) Then:
\[
\begin{split}
    \mathbb{E}[R_2] &= \sum\limits_{a} \Delta_a \mathbb{E}[T_a] = \sum\limits_{\Delta_a < x} \Delta_a \mathbb{E}[T_a] + \sum\limits_{\Delta_a \geq x} \Delta_a \mathbb{E}[T_a] \\&\leq Tx + \frac{8\alpha \log(T)}{x} S + \frac{4K\alpha n}{\alpha - 1} = 2\sqrt{8\alpha T \log(T) S} + \frac{4K\alpha n}{\alpha - 1}.
\end{split}
\]

Finally, for the total regret \( R = R_1 + R_2 \), we have:
\begin{align}
    \mathbb{E}[R] &\leq 2\sqrt{8\alpha T \log(T) S} + \frac{4K\alpha n}{\alpha - 1} + \sum\limits_{a} T_{1,a} \nonumber \\
    &\leq 2\sqrt{8\alpha T \log(T) S} 
    + \frac{4K\alpha n}{\alpha - 1} 
    + n \log(T)^2 \nonumber \\
    &= O\left( \sqrt{\alpha T \log(T) S} \right).
\end{align}

\end{proof}

\theomaruppersecond*

\begin{proof}
We follow the same approach as the previous proof. From the previous result, we know the following inequality holds with probability \( 1 - 2t^{-\alpha} \):
\[
\left| \sum\limits_{m \in \mathbb{M}} p_{m,a} \hat{\mu}_{m, a} - \mu_a \right| \leq \sqrt{\frac{\alpha \log (t)}{2} \sum\limits_{m \in \mathbb{M}} \frac{p_{m,a}^2}{T_{m, a, o}}}.
\]
Let this ``good event" be denoted as \( A_{t, a} \).

Now, we consider:
\begin{align}
    \label{eq:mar_proof2_1}
    \left| \sum\limits_{m \in \mathbb{M}} \hat{p}_{m, a} \hat{\mu}_{m, a} - \mu_a \right| 
    &= \left| \sum\limits_{m \in \mathbb{M}} (\hat{p}_{m, a} - p_{m,a}) \hat{\mu}_{m, a} + \sum\limits_{m \in \mathbb{M}} p_{m,a} \hat{\mu}_{m, a} - \mu_a \right| \\
    &\leq \sum\limits_{m \in \mathbb{M}} \left| \hat{p}_{m, a} - p_{m,a} \right| \hat{\mu}_{m, a} + \left| \sum\limits_{m \in \mathbb{M}} p_{m,a} \hat{\mu}_{m, a} - \mu_a \right|.
\end{align}

From the previous proof, since \( \hat{p}_{m, a} = \frac{s_{m,a}}{T_a} \), we have:
\[
\left| p_{m,a} - \hat{p}_{m, a} \right| \leq \sqrt{\frac{\alpha \log(t)}{2T_a}},
\]
with probability at least \( 1 - 2t^{-\alpha} \). Denote this "good event" as \( B_{t, i, a} \). Under this event for \( T_a \geq \log(T)^2\) and sufficient big \( T \) we will have \(  \frac{p_{m, a}}{2} \leq \hat{p}_{m, a} \leq 2 p_{m, a}\)

Additionally, we have:
\[
\left| \gamma_{i,a} - \frac{n_{i,a}}{s_{m,a}} \right| \leq \sqrt{\frac{\alpha \log(t)}{2s_{m,a}}} \leq \sqrt{\frac{\alpha \log(T)}{2s_{m,a}}},
\]
again with probability \( 1 - 2t^{-\alpha} \), denoted as "good event" \( C_{t, i, a} \).

If these ``good events" hold, we know:
\[
T_{m, a, o} \geq \frac{\gamma_{m, a} p_{m,a} T_a}{4}.
\]

We also know:
\[
\left| \mu_{i,a} - \hat{\mu}_{m, a} \right| \leq \sqrt{\frac{\alpha \log(t)}{2T_{m, a, o}}},
\]
which leads to:
\[
\hat{\mu}_{m, a} \leq \mu_{m, a} + \sqrt{\frac{\alpha \log(t)}{2T_{m, a, o}}} \leq 1 + \sqrt{\frac{\alpha \log(t)}{2T_{m, a, o}}},
\]
with probability at least \( 1 - 2t^{-\alpha} \), denoted as "good event" \( D_{t, i, a} \).

Under all these ``good events," we have:
\begin{align*}
    \sum\limits_{m \in \mathbb{M}} \left| \hat{p}_{m, a} - p_{m,a} \right| \hat{\mu}_{m, a} 
    &\leq \sum\limits_{m \in \mathbb{M}} \left| \hat{p}_{m, a} - p_{m,a} \right| + \sum\limits_{m \in \mathbb{M}} \left| \hat{p}_{m, a} - p_{m,a} \right| \times \sqrt{\frac{\alpha \log(t)}{2T_{m, a, o}}} \\
    &\leq \sum\limits_{m \in \mathbb{M}} \left| \hat{p}_{m, a} - p_{m,a} \right| + \sum\limits_{m \in \mathbb{M}} \sqrt{\frac{\alpha \log(T)}{2T_a}} \times \sqrt{\frac{2\alpha \log(T)}{p_{m,a} \gamma_{m, a} T_a}}.
\end{align*}

Using Lemma 7 from \cite{kamath2015learning}, we get:
\[
\sum\limits_{m \in \mathbb{M}} \left| \hat{p}_{m, a} - p_{m,a} \right| \leq \sqrt{\frac{2 (k - 1)}{\pi T_a}} + \frac{4k^{\frac{1}{2}} (k - 1)^{\frac{1}{4}}}{T_a^{\frac{3}{4}}}.
\]

Now, combining everything with the initial inequality \eqref{eq:mar_proof2_1}, we have:
\[
\begin{split}
    \left| \sum\limits_{m \in \mathbb{M}} \hat{p}_{m, a} \hat{\mu}_{m, a} - \mu_a \right| &\leq 
    \sqrt{\frac{2 (k - 1)}{\pi T_a}} + \frac{4k^{\frac{1}{2}} (k - 1)^{\frac{1}{4}}}{T_a^{\frac{3}{4}}} \\&+ \sum\limits_{m \in \mathbb{M}} \sqrt{\frac{\alpha \log(T)}{2T_a}} \times \sqrt{\frac{2\alpha \log(T)}{p_{m,a} \gamma_{m, a} T_a}} + \sqrt{\frac{\alpha \log(T)}{2} \sum\limits_{m \in \mathbb{M}} \frac{p_{m,a}^2}{T_{m, a, o}}}.
\end{split}
\]

For sufficiently large \( T \), since \( T_a > \log(T)^2 \), we have:
\[
\sqrt{\frac{\alpha \log(T)}{2} \sum\limits_{m \in \mathbb{M}} \frac{p_{m,a}^2}{T_{m, a, o}}} \geq \sum\limits_{m \in \mathbb{M}} \sqrt{\frac{\alpha \log(T)}{2T_a}} \times \sqrt{\frac{2\alpha \log(T)}{p_{m,a} \gamma_{m, a} T_a}}.
\]

and
\[
\sqrt{\frac{\alpha \log(T)}{2} \sum\limits_{m \in \mathbb{M}} \frac{p_{m,a}^2}{T_{m, a, o}}} \geq \sqrt{\frac{2 (k - 1)}{\pi T_a}} \text{ and } \frac{4k^{\frac{1}{2}} (k - 1)^{\frac{1}{4}}}{T_a^{\frac{3}{4}}}
\]

Thus, we conclude:
\[
\left| \sum\limits_{m \in \mathbb{M}} \hat{p}_{m, a} \hat{\mu}_{m, a} - \mu_a \right| \leq 4\sqrt{\frac{\alpha \log(T)}{2} \sum\limits_{m \in \mathbb{M}} \frac{p_{m,a}^2}{T_{m, a, o}}}.
\]

Finally, since \( \frac{p_{m,a}}{2} \leq \hat{p}_{m, a} \), we have:
\[
\left| \sum\limits_{m \in \mathbb{M}} \hat{p}_{m, a} \hat{\mu}_{m, a} - \mu_a \right| \leq 8\sqrt{\frac{\alpha \log(T)}{2} \sum\limits_{m \in \mathbb{M}} \frac{\hat{p}_{m, a}^2}{T_{m, a, o}}}.
\]

Following the exact reasoning in the proof of Theorem~\ref{theo:mar_upper_1}, and using the fact that \( \hat{p}_{m, a} \leq 2p_{m, a} \) we conclude:
\[
\mathbb{E}[R] = O\left( \alpha \sqrt{T \log(T) S} \right).
\]   
\end{proof}

\theomarlower*
\begin{proof}
Consider the following $n + 1$ bandit instances, with $n$ arms labeled $a_1, a_2, \dots, a_n$.

\textbf{Bandit instance 0:}
\begin{itemize}[noitemsep, topsep=0pt]
    \item $\mu_{m, a} = 0$ for all arms $a = a_1, \dots, a_n$ and for all $m =  1, \dots, K$.
\end{itemize}

\textbf{Bandit instance $j$ for $j = 1, \dots, n$:}
\begin{itemize}[noitemsep, topsep=0pt]
    \item $\mu_{m, a} = \frac{\Delta}{P_a \gamma_{m, a}}$ for arm $a = a_j$, and for all $m =  1, \dots, K$.
    \item $\mu_{m, a} = 0$ for all arms $a \neq a_j$, and for all $m =  1, \dots, K$.
\end{itemize}

\vspace{1em} 

\textbf{For each instance \( j \in \{1, \dots, n\}\):}
\begin{itemize}[noitemsep, topsep=0pt]
    \item If $a = j$: \( \mu_a = \sum_{m \in [K]} p_{m,a} \mu_{m, a} = \Delta \).
    \item If $a \neq j$: \( \mu_a = 0 \).
\end{itemize}

\vspace{0.75em} 

\textbf{For instance 0:}
\begin{itemize}[noitemsep, topsep=0pt]
    \item For all $a \in [1, \dots, n]$: \( \mu_a = 0 \).
\end{itemize}

Now like the previous we use the mentioned lemmas from \cite{lattimore2020bandit} to complete our proof.

Like before for every \( a = 1, \dots, n \) we have:
\begin{align*}    
\text{KL}(P_0, P_a) &= \mathbb{E}_0 [T_a] \text{KL}(P_0(a), P_a(a)) \\
&= \mathbb{E}_0 [T_a] \sum\limits_{m \in [K]} p_{m,a} \gamma_{m, a} \frac{\Delta^2}{2P_a^2 \gamma_{m, a}^2} \\
&= \mathbb{E}_0 [T_a] \frac{\Delta^2}{P_a^2} \sum\limits_{m \in [K]} \frac{p_{m,a}}{\gamma_{m, a}} \\
&= \mathbb{E}_0 [T_a] \frac{\Delta^2}{P_a}
\end{align*}

From this, we can bound $\mathbb{E}_a [T_a(T)]$ as follows:
\begin{align*}
\mathbb{E}_a [T_a(T)] &\leq \mathbb{E}_0 [T_a(T)] + T d_\text{TV}(P_0(a), P_a(a)) \\
&\leq \mathbb{E}_0 [T_a(T)] + T \sqrt{\frac{1}{2} \text{KL}(P_0(a), P_a(a))} \\
&= \mathbb{E}_0 [T_a(T)] + \frac{T}{2} \sqrt{\mathbb{E}_0 [T_a] \frac{\Delta^2}{P_a}}
\end{align*}

Let $R_m =  R_T(\pi; i)$ denote the regret of applying policy $\pi$ on the $i$-th bandit instance up to time $T$, where $i$ refers to the $i$-th bandit instance.

Summing over all bandit instances \(1, \dots, n\), we have:
\begin{align*}
    \sum_{i=1}^n \mathbb{E}[R_i] 
    &= \sum\limits_{a \in [n]} \Delta \left(T - \mathbb{E}_a[T_a(T)]\right) \\
    &\geq \Delta Tn - \Delta \sum\limits_{a \in [n]} \left( \mathbb{E}_0 [T_a(T)] + \frac{T}{2} \sqrt{\mathbb{E}_0 [T_a] \frac{\Delta^2}{P_a}} \right) \\
    &\geq \Delta Tn - \Delta T - \frac{\Delta^2 T}{2} \sum\limits_{a \in [n]} \sqrt{\frac{\mathbb{E}_0 [T_a(T)]}{P_a}} \\
    &\geq \Delta Tn - \Delta T - \frac{\Delta^2 T}{2} \sqrt{T \sum\limits_{a \in [n]} \frac{1}{P_a}} \\
    &\geq \frac{\Delta Tn}{2} - \frac{\Delta^2 T}{2} \sqrt{T \sum\limits_{a \in [n]} \frac{1}{P_a}} \quad \text{using } \Delta = \frac{n}{2\sqrt{T\sum\limits_{a \in [n]} \frac{1}{P_a}}} \\
    &\geq \frac{n}{8} \sqrt{\frac{Tn^2}{\sum\limits_{a \in [n]} \frac{1}{P_a}}}
\end{align*}

Thus, there exists an instance where \( \mathbb{E}[R_i] \geq \Omega\left( \sqrt{\frac{Tn^2}{\sum\limits_{a \in [n]} \frac{1}{P_a}}} \right) = \Omega\left( \sqrt{TnH} \right)  \).
\end{proof}

\thmignoremed*
\begin{proof}
    We construct a bandit with two arms such that the KL divergence of the outputs observed by the agents, when there is no mediator, is zero. In this scenario, we have:

\[
\mathbb{E}[T_1] \leq \mathbb{E}[T_2] + T \sqrt{\mathbb{E}[T_1 \text{KL}(P_1, P_2)]} = \mathbb{E}[T_2],
\]

and

\[
\mathbb{E}[T_2] \leq \mathbb{E}[T_1] + T \sqrt{\mathbb{E}[T_2 \text{KL}(P_2, P_1)]} = \mathbb{E}[T_1],
\]

which implies:

\[
\mathbb{E}[T_1] = \mathbb{E}[T_2] = \frac{T}{2}.
\]

Thus, if the actual means of the arms differ, with \( \mu_1 - \mu_2 = \Delta \), then:

\[
\mathbb{E}[R_T] = \Omega\left(\frac{T \Delta}{2}\right) = \Omega(T).
\]

Now, let \( f_Y(y) \) represent the probability mass function of \( Y \). We have:

\[
f_Y(y \mid a, O^Y = 1) = \sum\limits_{m} \mathbb{P}(m \mid a, O^Y = 1) f_Y(y \mid a, m, O^Y = 1) = \sum\limits_{m} \mathbb{P}(m \mid a, O^Y = 1) f_Y(y \mid a, m),
\]

and

\[
\mathbb{P}(m \mid a, O^Y = 1) = \frac{\mathbb{P}(m, a, O^Y = 1)}{\sum\limits_{m} \mathbb{P}(m, a, O^Y = 1)} = \frac{\gamma_{m, a} p_{m, a}}{\sum\limits_{m} \gamma_{m, a} p_{m, a}}.
\]

Now, if we let \( \gamma_{m, a} = \frac{\frac{1}{p_{m, a}}}{\sum\limits_{m} \frac{1}{p_{m, a}}} \), we have:

\[
f_Y(y \mid a, O^Y = 1) = \frac{\sum\limits_{m} f_Y(y \mid a, m)}{K},
\]

This expression is independent of \( p_{m, a} \). Additionally, we have:

\[
\mathbb{P}(O^Y = 1 \mid a) = \sum\limits_{m} p_{m, a} \gamma_{m, a} = \frac{K}{\sum\limits_{m} \frac{1}{p_{m, a}}},
\]

which leads to \( \text{KL}(P_1, P_2) = \text{KL}(P_2, P_1) = 0 \) for any choice of \( p_{m, a} \) such that the set \( P_a = \{ p_{m, a} \mid \forall m \in \mathbb{M} \} \) is the same for both arms.

Since \( \mu_a = \sum\limits_{m} p_{m, a} \mu_{m, a} \), we let all \( \mu_{m, a} \) be zero except for one, which we set to 1. Now, for arm \( a = 1 \), let \( p_{m, a} = 1 - \epsilon \) for the \( m \) such that \( \mu_{m, a} = 1 \), and set the others equal to \( \frac{\epsilon}{K - 1} \). For arm \( a = 2 \), let \( p_{m, a} = 1 - \epsilon \) for the \( m \) such that \( \mu_{m, a} \neq 1 \), and set the others equal to \( \frac{\epsilon}{K - 1} \).

In this way, \( \mu_1 - \mu_2 = 1 - \epsilon - \frac{\epsilon}{K - 1} \), completing the construction for small \( \epsilon \).

\end{proof}

\theomnarupper*
\newcommand{\norm}[1]{\| #1 \|}
\begin{proof}
We follow the same approach used in previous upper bound proofs to derive a lower bound on the estimation error of \( \mu_a \).

At each time step \( 1 \leq t \leq T_2 \), define:

\begin{align*}
&b_a = [\mathbb{P}(m, O^Y = 0 \mid a)]_{K \times 1} \\
&\Theta_a = [\mathbb{P}(m, y, O^Y = 1 \mid a)]_{K \times L} \\
&x_a = [\frac{\mathbb{P}(O^Y = 0 \mid y, a)}{\mathbb{P}(O^Y = 1 \mid y, a)}]_{L \times 1} \\
\end{align*}

Since we know that \( \Theta_a x_a = b_a \), we now invoke Theorem 2.2 from \cite{higham1994survey}, which states:

\textbf{Theorem 2.2.} Let \(Ax = b\) and \((A + \Delta A) y = b + \Delta b\), where \(\|\Delta A\| \leq \epsilon \|E\|\) and \(\|\Delta b\| \leq \epsilon \|f\|\), and assume that \(\epsilon \|A^{-1}\| \|E\| < 1\). Then:

\begin{equation}
    \label{eq:mnar_survey_theo}
    \frac{\|x - y\|}{\|x\|} \leq \frac{\epsilon}{1 - \epsilon \|A^{-1}\| \|E\|}
    \left(
    \frac{\|A^{-1}\| \|f\|}{\|x\|} + \|A^{-1}\| \|E\|
    \right),
\end{equation}

and this bound is attainable to first order in \(\epsilon\).

For each entry of \( b_a \) or \( \Theta_a \), we have \( T_a \) samples. By applying Hoeffding's inequality and following the same approach as in the proofs of previous theorems, we set \( \epsilon = \sqrt{\frac{\alpha \log(T)}{2 T_a}} \). Consequently, we obtain the following bounds (all norms are \( \|.\|_\infty \)):

\begin{align*}
    \| \hat{b}_a - b_a \| &\leq \epsilon, \\
    \| \hat{\Theta}_a - \Theta_a \| &\leq L \epsilon,
\end{align*}

with probability at least \( 1 - 2 K \times (L + 1) t^{-\alpha} \).

Now, under the event described above and using \eqref{eq:mnar_survey_theo}, we have:

\[
    \frac{\|x_a - \hat{x}_a\|}{\|x_a\|} \leq \frac{\epsilon}{1 - \epsilon L \| \Theta_a^{-1}\|}
    \left(
    \frac{\|\Theta_a^{-1}\|}{\|x_a\|} + L \|\Theta_a^{-1}\|
    \right).
\]

We have \( x_a = \left[\frac{1 - \gamma_{y,a}}{\gamma_{y,a}}\right]_{L \times 1} \), and therefore \( \|x_a\| = \frac{1 - \gamma_a}{\gamma_a} \). For sufficiently large \( T \) and for \( T_a \geq \log(T)^2 \), it follows that \( \epsilon L \|\Theta_a^{-1}\| \leq \frac{1}{2} \), leading to:

\[
    \|x_a - \hat{x}_a\| 
    \leq 2 \epsilon 
    \left(
    \|\Theta_a^{-1}\| + L \|\Theta_a^{-1} \| \frac{1 - \gamma_a}{\gamma_a}
    \right)
    = 2 \epsilon \|\Theta_a^{-1}\|
    \left(
    \frac{L}{\gamma_a} - (L - 1)
    \right)
    \leq 2 \epsilon \frac{L}{\gamma_a} \|\Theta_a^{-1}\|.
\]

Now, since \( \|\hat{\Theta}_a - \Theta_a\| \leq L \epsilon \), for sufficiently large \( T \) and \( T_a \geq \log(T)^2 \), we have \( L \epsilon \leq \frac{\|\Theta_a\|}{2} \). Hence:

\[
\frac{\|\Theta_a\|}{2} \leq \|\hat{\Theta}_a\| \leq 2 \|\Theta_a\|,
\]

which implies

\[
\|\Theta_a^{-1}\| = \frac{\kappa(\Theta_a)}{\|\Theta_a\|} \leq \frac{2\kappa(\Theta_a)}{\|\hat{\Theta}_a\|} \leq \frac{2C_a}{\|\hat{\Theta}_a\|}.
\]

Thus:

\begin{equation}
\label{eq:mnar_eq2}
\norm{x_a - \hat{x}_a} \leq  4 \epsilon \frac{L C_a}{\gamma_a \norm{\hat{\Theta}_a}}.
\end{equation}

For sufficiently large \( T \) and \( T_a \geq \log(T)^2 \), we will have:
\[
\norm{x_a - \hat{x}_a} \leq \frac{1}{2\gamma_a}
\]
so:

\[
\frac{1}{\hat{\gamma}_a} = \norm{\hat{x}_a + [1]_{L \times 1}} \geq \norm{x_a + [1]_{L \times 1}} - \frac{1}{2\gamma_a} = \norm{[\frac{1}{\gamma_{y, a}}]_{L \times 1}} - \frac{1}{2\gamma_a} = \frac{1}{2\gamma_a}.
\]

Using \eqref{eq:mnar_eq2}, we have:
\[
\norm{x_a - \hat{x}_a} \leq  8 \epsilon \frac{L C_a}{\hat{\gamma}_a \norm{\hat{\Theta}_a}}.
\]

Since \( x_a + [1]_{L \times 1} = [\frac{1}{\gamma_{y, a}}]_{L \times 1} \), for every \( y \), we have:
\[
\left|\frac{1}{\gamma_{y, a}} - \frac{1}{\hat{\gamma}_{y, a}}\right| \leq  8 \epsilon \frac{L C_a}{\hat{\gamma}_a \norm{\hat{\Theta}_a}}.
\]

Now let \( p_{m, y \mid 1, a}  = \mathbb{P}(m, y \mid O^y = 1, a) \). By applying Hoeffding's inequality, we have the following inequality for all \( m, y \):

\[
    |q_{m, y \mid 1, a} - p_{m, y \mid 1, a}| \leq \sqrt{\frac{\alpha \log(T)}{2 T_{a, o}}}
\]
with probability at least \(1 - 2K L t^{-\alpha} \).
Using the fact that \( \frac{p_{m, y \mid 1, a}}{\gamma_{y, a}} = \mathbb{P}(m, y \mid a) = p_{m, y \mid a} \), we have:
\begin{align*}
\left|p_{m, y \mid a} - \hat{p}_{m, y \mid a}\right|
&= \left| \frac{p_{m, y \mid a}}{\gamma_{y, a}} - \frac{q_{m, y \mid a}}{\hat{\gamma_{y, a}}} \right| \\
&\leq  p_{m, y \mid a} \left|\frac{1}{\gamma_{y, a}} - \frac{1}{\hat{\gamma}_{y, a}}\right| + \frac{1}{\hat{\gamma}_{a}} \left|q_{m, y \mid 1, a} - p_{m, y \mid 1, a}\right| \\
&\leq  8 p_{m, y \mid a} \epsilon \frac{L C_a}{\hat{\gamma}_a \norm{\hat{\Theta}_a}} + \frac{1}{\hat{\gamma}_{a}} \sqrt{\frac{\alpha \log(T)}{T_{a, o}}}.
\end{align*}

Summing up over \( m \), we have:
\begin{align*}
\left|p_{y \mid a} - \hat{p}_{y \mid a}\right|
 &\leq 8 p_{y \mid a} \epsilon \frac{L C_a}{\hat{\gamma}_a \norm{\hat{\Theta}_a}} + \frac{K}{\hat{\gamma}_{a}} \sqrt{\frac{\alpha \log(T)}{T_{a, o}}} \\
 &\leq 8 \epsilon \frac{L C_a}{\hat{\gamma}_a \norm{\hat{\Theta}_a}} + \frac{K}{\hat{\gamma}_{a}} \sqrt{\frac{\alpha \log(T)}{T_{a, o}}}.
\end{align*}

Thus, using \( \sum\limits_{y} |y| = 1\):

\[
| \mu_a - \hat{\mu}_a| \leq \sum\limits_{y} |y| \left|p_{y \mid a} - \hat{p}_{y \mid a}\right| \leq \sum\limits_{y} |y| \left( 8 \epsilon \frac{L C_a}{\hat{\gamma}_a \norm{\hat{\Theta}_a}} + \frac{K}{\hat{\gamma}_{a}} \sqrt{\frac{\alpha \log(T)}{T_{a, o}}} \right) = 8 \epsilon \frac{L C_a}{\hat{\gamma}_a \norm{\hat{\Theta}_a}} + \frac{K}{\hat{\gamma}_{a}} \sqrt{\frac{\alpha \log(T)}{T_{a, o}}} = \epsilon_a.
\]

Hence, \( \text{UCB}(a) = \hat{\mu}_a + \epsilon_a \), and using previous proofs, we conclude that \( \epsilon_a \geq \frac{\Delta_a}{2} \). To finalize the proof:

\[
\norm{\hat{\Theta}_a} \geq \frac{\norm{\Theta_a}}{2}, 
\hat{\gamma}_a \geq  \frac{\gamma_a}{2},
\]

and we have \( \mathbb{P}(O^Y = 1 \mid a) = \sum\limits_{y} p_{y,a}\gamma_{y,a} \). Applying Hoeffding's inequality gives:

\[
\left|\frac{T_{a, o}}{T_a} - \sum\limits_{y} p_{y,a}\gamma_{y,a}\right| \leq \epsilon,
\]

which for sufficiently large \( T \) and \( T_a \geq \log(T)^2 \), states:

\[
    T_{a, o} \geq T_a \left(\sum\limits_{y} p_{y,a}\gamma_{y,a}\right).
\]

Finally, we have:

\[
\epsilon_a \leq \sqrt{\frac{\alpha \log(T)}{T_a}} 
8\frac{L C_a}{\gamma_a \norm{\Theta_a}} + 2\frac{K}{\gamma_{a}} \sqrt{\frac{1}{\sum\limits_{y} p_{y,a}\gamma_{y,a}}}
\leq 8\sqrt{\frac{\alpha \log(T)}{T_a}} 
\max\left( \frac{L C_a}{\gamma_a \norm{\Theta_a}}, \frac{K}{\gamma_{a}} \sqrt{\frac{1}{\sum\limits_{y} p_{y,a}\gamma_{y,a}} }\right),
\]

which, following the exact steps of previous theorem proofs, leads to:

\[
\mathbb{E}[R_T] = O\left( \sqrt{\alpha T \log(T) \sum\limits_{a} S_a^2} \right).
\]
\end{proof}

\section{Theoretical results on Missing Outcome and Missing Mediator}\label{apx:Missing M}
In this section, we present our theoretical results on Missing at Random (MAR) and Missing Not at Random (MNAR) environments for both Missing Outcome and Missing Mediator cases.

\subsection{Missing at Random (MAR)}
As discussed earlier, the identification of \( \mu_a = \mathbb{E}[Y \mid a] \) is given by:

\[
\mu_a = \sum\limits_{m \in \mathbb{M}} \mathbb{P}(M = m \mid a, O^M = 1) \mathbb{E}[Y \mid M = m, a, O^Y = 1, O^M = 1].
\]

Using this identification, we will prove the following theorem. We define \( p_{m, a} = \mathbb{P}(M = m, a), \gamma_{m, a} = \mathbb{P}(O^Y = 1 \mid M = m, a), \lambda_a = \mathbb{P}(O^M = 1 \mid a) \). Our algorithm is exactly like \ref{alg:mar_algorithm2} where if \( M \) is missed we don't update anything.

\begin{restatable}{theorem}{theoMmismarsupper} \label{theo:M_miss_upper}
\text{(Regret bound for MAR with missing mediator and outcome)}
For every \( \alpha > 1 \), the following regret bound holds for sufficiently large \( T \):
\[
\mathbb{E}[R_T] = O\left( \sqrt{\alpha T \log(T) n S} \right),
\]
where \( P_a = \sum_{m \in \mathcal{M}} \frac{p_{m,a}}{\gamma_{m,a} \lambda_a} \) and \( S \coloneqq \frac{1}{\vert \mathcal{A}\vert} \sum\limits_{a\in\mathcal{A}} P_a \).
\end{restatable}

\begin{proof}
    The proof follows a similar approach to the proof of Theorem \ref{theo:mar_upper_2}. Using the same reasoning, we have (where \( T_{m, a, o_Y} \) is the number of times \( M = m \) and the reward are observed when pulling arm \( a \)):
    \[
    \left| \sum\limits_{m \in \mathbb{M}} \hat{p}_{m, a} \hat{\mu}_{m, a} - \mu_a \right| \leq 8\sqrt{\frac{\alpha \log(T)}{2} \sum\limits_{m \in \mathbb{M}} \frac{\hat{p}_{m, a}^2}{T_{m, a, o_Y}}}.
    \]
    Similarly, we also have the following inequality (where \( T_{a, o_M} \) is the number of times \( M = m \) is observed when pulling arm \( a \)):
    \[
    \left| p_{m,a} - \hat{p}_{m, a} \right| \leq \sqrt{\frac{\alpha \log(t)}{2T_{a, o_M}}}.
    \]
    Additionally, we have:
    \[
    \left| \frac{T_{a, o_M}}{T_a} - \lambda_a \right| \leq \sqrt{\frac{\alpha \log(t)}{2T_a}}.
    \]
    Under this event, for \( T_a \geq \log(T)^2 \) and sufficiently large \( T \), we have \( \frac{p_{m, a}}{2} \leq \hat{p}_{m, a} \leq 2 p_{m, a} \). Following similar steps, we get:
    \[
    T_{m, a, o_Y} \geq \frac{p_{m, a} \gamma_{m, a} \lambda_a T_a}{4}.
    \]
    Therefore, using the same definition of \( \epsilon_a \), we obtain:
    \[
    \epsilon_a \leq 8\sqrt{\frac{8\alpha \log(T)}{2} \sum\limits_{m \in \mathbb{M}} \frac{p_{m, a}^2}{p_{m, a} \gamma_{m, a} \lambda_a T_a}} = 8\sqrt{\frac{8\alpha \log(T)}{2} \sum\limits_{m \in \mathbb{M}} \frac{p_{m, a}}{\gamma_{m, a} \lambda_a T_a}}.
    \]
    Finally, following the same steps as in previous proofs, we conclude:
    \[
    \mathbb{E}[R_T] = O\left( \sqrt{\alpha T \log(T) n S} \right).
    \]
    
\end{proof}

\subsection{Missing Not at Random (MNAR)}
In this section, we use the identification formula discussed earlier to develop an algorithm and establish an upper bound for this environment. Assume that  
\( \mathbb{P}(m \mid a) = p_{m, a} \), \( \mathbb{P}(O^M = 1 \mid m, a) = \lambda_{m, a} \), \( \mathbb{P}(O^Y = 1 \mid m, a) = \gamma_{m, a} \). Also define \( \lambda_a = \min\limits_{m} \lambda_{m, a} \). We assume a similar condition to Assumption~\ref{ass:mnar_K_assumption}, but for a different matrix. Let \( \Theta_a = [\mathbb{P}(m, y, O^M = 1, O^Y = 1 \mid a)]_{K \times L} \):

\begin{assumption}[Bounded condition number]\label{ass:missing_M_mnar_K_assumption}
For each arm \( a \in \mathcal{A} \), the condition number of the matrix \( \Theta_a \) is bounded by:
\[
    \kappa(\Theta_a) \leq C_a,
\]
where \( \kappa(\Theta_a) \) denotes the condition number of \( \Theta_a \) with respect to the $\infty$-norm, defined as 
\[
    \kappa(\Theta_a) = \lVert \Theta_a \rVert_\infty \lVert \Theta_a^{\dagger} \rVert_\infty,
\] with $\Theta_a^{\dagger}$ being the pseudo-inverse of \( \Theta_a \).
\end{assumption}

In our algorithm we use the the given identification formula and 
\[ 
\text{UCB}(a) = \hat{\mu}_a + 2 \sqrt{\frac{\alpha \log(T)}{2T_a}} \left( 8 \frac{C_a}{\norm{\hat{\Theta}_a}} \frac{K}{\hat{\lambda}_a} + \frac{1}{\hat{\lambda}_a} \right) + \sqrt{\frac{\alpha \log(T)}{2} \sum\limits_{m \in \mathbb{M}} \frac{4 \hat{p}_{m, a}^2}{T_{a, m, o_Y}}}.
\]

we will prove the following theorem.
\begin{restatable}{theorem}{theoMmismnarsupper} \label{theo:M_miss_mnar_upper}
\text{(Regret bound for MNAR with missing mediator and outcome)}
For every \( \alpha > 1 \), under Assumption~\ref{ass:missing_M_mnar_K_assumption}, the following regret bound holds for sufficiently large \( T \):
\[
    \mathbb{E}[R_T] = O\left( \sqrt{\alpha T \log(T) \sum\limits_{a} S_a^2} \right).
\]
where \( S_a = \max 
    \left(
    \left( 32 \frac{C_a}{\norm{\Theta}_a} \frac{K}{\lambda_a} + \frac{2}{\lambda_a} \right), 
    \sqrt{\sum\limits_{m \in \mathbb{M}} \frac{32 p_{m, a}}{\lambda_{m, a} \gamma_{m, a}}}
    \right) 
    \).
\end{restatable}

\begin{proof}
    The proof closely follows the reasoning from Theorem~\ref{theo:mnar_upper}. Let:
    \begin{align*}
        &b_a = [\mathbb{P}(y, O^M = 0, O^Y = 1 \mid a)]_{L \times 1}, \\
        &\Theta_a = [\mathbb{P}(m, y, O^M = 1, O^Y = 1 \mid a)]_{K \times L}, \\
        &x_a = \left[\frac{\mathbb{P}(O^M = 0 \mid m, a)}{\mathbb{P}(O^M = 1 \mid m, a)}\right]_{K \times 1}.
    \end{align*}
    We know that \( \Theta_a x_a = b_a \). Using the same approach as in Theorem~\ref{theo:mnar_upper}, we derive the following inequality for \( \epsilon = \sqrt{\frac{\alpha \log(T)}{2T_a}} \):
    \[
        \norm{x - \hat{x}} \leq 2 \epsilon \norm{\Theta_a^{-1}} \left(1 + K \frac{1 - \lambda_a}{\lambda_a}\right) \leq 2 \epsilon \norm{\Theta_a^{-1}} \frac{K}{\lambda_a} \leq 4 \epsilon \frac{C_a}{\norm{\hat{\Theta}_a}} \frac{K}{\lambda_a} \leq 8 \epsilon \frac{C_a}{\norm{\hat{\Theta}_a}} \frac{K}{\hat{\lambda}_a}.
    \]
    Additionally, since \( x = \left[\frac{1 - \lambda_{m, a}}{\lambda_{m, a}}\right]_{K \times 1} \), we have:
    \[
        \left|\frac{1}{\lambda_{m, a}} - \frac{1}{\hat{\lambda}_{m, a}}\right| \leq 8 \epsilon \frac{C_a}{\norm{\hat{\Theta}_a}} \frac{K}{\hat{\lambda}_a}.
    \]
    
    Furthermore, for \( p_{m, 1 \mid a} = \mathbb{P}(M = m, O^M = 1 \mid a) \), we have:
    \[
        \left| p_{m, 1 \mid a} - \hat{p}_{m, 1 \mid a} \right| \leq \epsilon.
    \]
    Using a similar approach to the proof of Theorem~\ref{theo:mnar_upper}, we obtain:
    \[
        \left|\frac{p_{m, 1 \mid a}}{\lambda_{m, a}} - \frac{\hat{p}_{m, 1 \mid a}}{\hat{\lambda}_{m, a}}\right| \leq p_{m, 1 \mid a} \left( 8 \epsilon \frac{C_a}{\norm{\hat{\Theta}_a}} \frac{K}{\hat{\lambda}_a} \right) + \frac{1}{\hat{\lambda}_a} \epsilon = \epsilon \left( 8 p_{m, 1 \mid a} \frac{C_a}{\norm{\hat{\Theta}_a}} \frac{K}{\hat{\lambda}_a} + \frac{1}{\hat{\lambda}_a} \right).
    \]
    
    Therefore, we can conclude:
    \[
    \left| \hat{p}_{m, a} - p_{m, a} \right| \leq \epsilon \left( 8 p_{m, 1 \mid a} \frac{C_a}{\norm{\hat{\Theta}_a}} \frac{K}{\hat{\lambda}_a} + \frac{1}{\hat{\lambda}_a} \right).
    \]
    
    Additionally, we have the following bound for \( T_{a, o_M, o_Y} \) (the number of times both \( M \) and \( Y \) are observed):
    \[
    \left| \hat{\mu}_{m, a} - \mu_{m, a} \right| \leq \sqrt{\frac{\alpha \log(T)}{T_{a, o_M, o_Y}}}.
    \]
    Using \( \mathbb{P}(O^Y = 1, O^M = 1 \mid a) = \sum\limits_{m \in \mathbb{M}} p_{m, a} \gamma_{m, a} \lambda_{m, a} \), we have:
    \[
    \left| \frac{T_{a, o_M, o_Y}}{T_a} - \sum\limits_{m \in \mathbb{M}} p_{m, a} \gamma_{m, a} \lambda_{m, a} \right| \leq \sqrt{\frac{\alpha \log(T)}{T_a}},
    \]
    which gives \( T_{a, o_M, o_Y} \geq \frac{T_a}{2} \left( \sum\limits_{m \in \mathbb{M}} p_{m, a} \gamma_{m, a} \lambda_{m, a} \right) \). Thus, for sufficiently large \( T \) and \( T_a \geq \log(T)^2 \), we have \( \hat{\mu}_{m, a} \leq 2 \mu_{m, a} \leq 2 \).
    
    Therefore:
    \begin{align}
    \left| \sum\limits_{m \in \mathbb{M}} \hat{p}_{m, a} \hat{\mu}_{m, a} - \mu_a \right| 
    &= \left| \sum\limits_{m \in \mathbb{M}} (\hat{p}_{m, a} - p_{m, a}) \hat{\mu}_{m, a} + \sum\limits_{m \in \mathbb{M}} p_{m, a} \hat{\mu}_{m, a} - \mu_a \right| \\
    &\leq \sum\limits_{m \in \mathbb{M}} \left| \hat{p}_{m, a} - p_{m, a} \right| \hat{\mu}_{m, a} + \left| \sum\limits_{m \in \mathbb{M}} p_{m, a} \hat{\mu}_{m, a} - \mu_a \right| \\
    &\leq 2 \sum\limits_{m \in \mathbb{M}} \epsilon \left( 8 p_{m, 1 \mid a} \frac{C_a}{\norm{\hat{\Theta}_a}} \frac{K}{\hat{\lambda}_a} + \frac{1}{\hat{\lambda}_a} \right) + \left| \sum\limits_{m \in \mathbb{M}} p_{m, a} \hat{\mu}_{m, a} - \mu_a \right| \quad \text{(using \( \sum\limits_{m \in \mathbb{M}} p_{m, 1 \mid a} \leq 1 \))} \\
    &\leq 2 \epsilon \left( 8 \frac{C_a}{\norm{\hat{\Theta}_a}} \frac{K}{\hat{\lambda}_a} + \frac{1}{\hat{\lambda}_a} \right) + \left| \sum\limits_{m \in \mathbb{M}} p_{m, a} \hat{\mu}_{m, a} - \mu_a \right|.
    \end{align}
    
    Using the same technique as before, we have the following inequality for \( T_{a, m, o_Y} \) (the number of times \( M = m \) and the reward are observed when pulling arm \( a \)):
    \[
    \left| \sum\limits_{m \in \mathbb{M}} p_{m, a} \hat{\mu}_{m, a} - \mu_a \right| \leq \sqrt{\frac{\alpha \log(T)}{2} \sum\limits_{m \in \mathbb{M}} \frac{4 \hat{p}_{m, a}^2}{T_{a, m, o_Y}}}.
    \]
    
    Therefore:
    \begin{align*}
    \left| \sum\limits_{m \in \mathbb{M}} \hat{p}_{m, a} \hat{\mu}_{m, a} - \mu_a \right| 
    &\leq 2 \epsilon \left( 8 \frac{C_a}{\norm{\hat{\Theta}_a}} \frac{K}{\hat{\lambda}_a} + \frac{1}{\hat{\lambda}_a} \right) + \sqrt{\frac{\alpha \log(T)}{2} \sum\limits_{m \in \mathbb{M}} \frac{4 \hat{p}_{m, a}^2}{T_{a, m, o_Y}}}
    \end{align*}
    which proves our UCB upper bound.
    
    Similarly, we know that \( T_{a, m, o_Y} \geq \frac{T_a}{2} p_{m, a} \lambda_{m, a} \gamma_{m, a} \), which gives:
    \begin{align*}
    |\mu_a - \hat{\mu_a}| 
    &\leq 2 \epsilon \left( 8 \frac{C_a}{\norm{\hat{\Theta}_a}} \frac{K}{\hat{\lambda}_a} + \frac{1}{\hat{\lambda}_a} \right) + \sqrt{\frac{\alpha \log(T)}{2} \sum\limits_{m \in \mathbb{M}} \frac{32 p_{m, a}^2}{T_a p_{m, a} \lambda_{m, a} \gamma_{m, a}}} \\
    &= 2 \epsilon \left( 8 \frac{C_a}{\norm{\hat{\Theta}_a}} \frac{K}{\hat{\lambda}_a} + \frac{1}{\hat{\lambda}_a} \right) + \sqrt{\sum\limits_{m \in \mathbb{M}} \frac{32 p_{m, a}}{\lambda_{m, a} \gamma_{m, a}}} \\
    &\leq 2\epsilon \max 
    \left(
    \left( 32 \frac{C_a}{\norm{\Theta}_a} \frac{K}{\lambda_a} + \frac{2}{\lambda_a} \right), 
    \sqrt{\sum\limits_{m \in \mathbb{M}} \frac{32 p_{m, a}}{\lambda_{m, a} \gamma_{m, a}}}
    \right).
    \end{align*}
    
    Following the same reasoning as in previous proofs, we conclude:
    \[
    \mathbb{E}[R_T] = O\left( \sqrt{\alpha T \log(T) \sum\limits_{a} S_a^2} \right).
    \]
\end{proof}

\clearpage
\section{Additional Empirical Evaluation}
\label{appendix:additional_simulations}

\begin{figure*}[h]
    \centering
    \begin{subfigure}[b]{0.38\textwidth}
        \includegraphics[width=\textwidth]{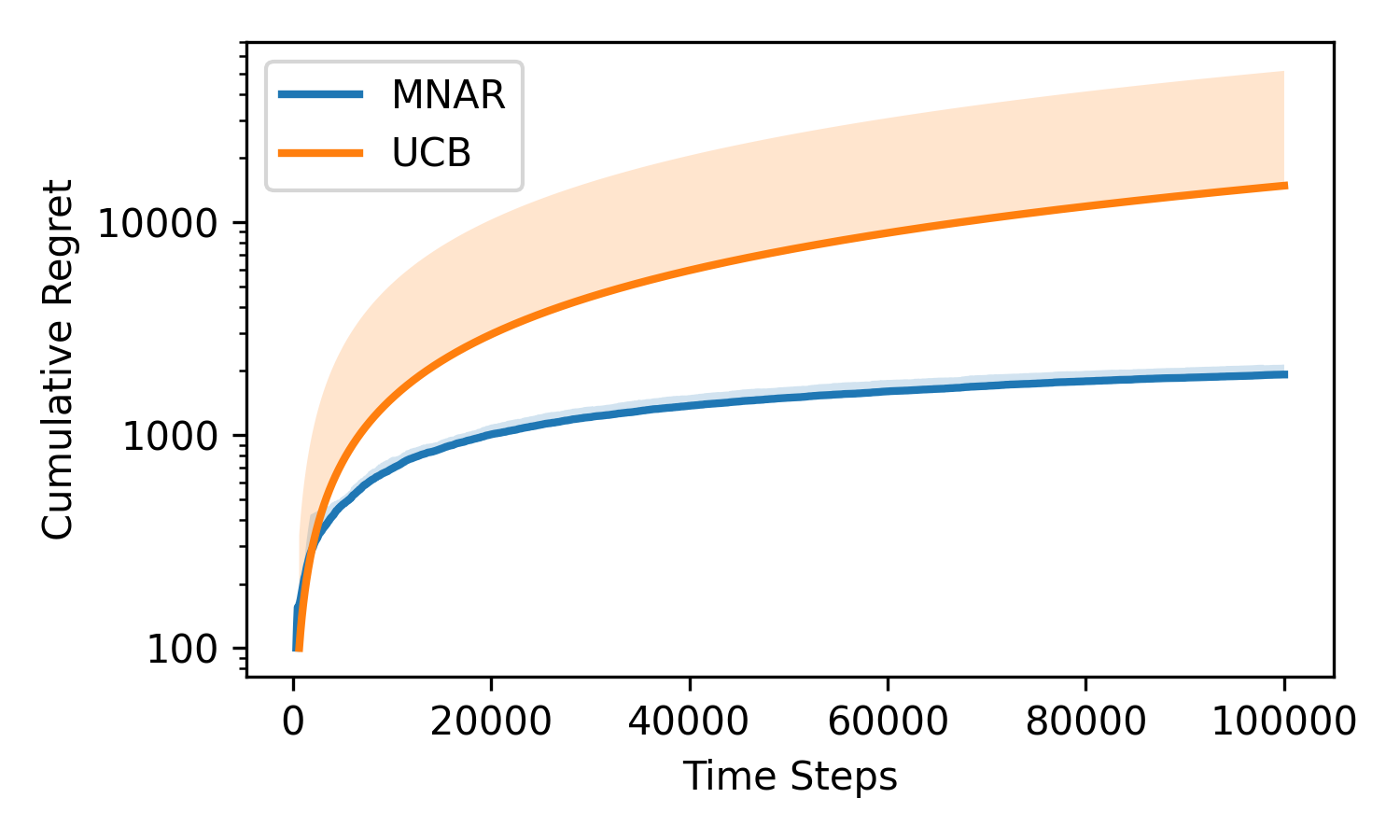}
        \caption{MNAR and UCB algorithms in the MNAR bandit environment.}
        \label{fig:MNAR-UCB}
    \end{subfigure}
    \begin{subfigure}[b]{0.38\textwidth}
        \includegraphics[width=\textwidth]{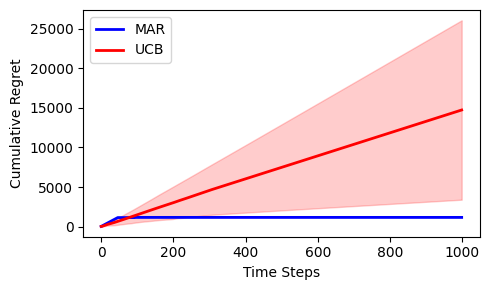}
        \caption{MAR and UCB algorithms in a real-world MAR bandit environment.}
        \label{fig:RW-1}
    \end{subfigure}
    \hfill
    \caption{Complementary evaluation results for our proposed algorithms.}
\end{figure*}

In Figure~\ref{fig:MNAR-UCB}, we compare the performance of the UCB and MNAR algorithms in the MNAR bandit environment. The results clearly demonstrate that the cumulative regret of the UCB algorithm is consistently higher than that of the MNAR algorithm. Additionally, the y-axis is displayed on a logarithmic scale, further highlighting the considerable difference in the performance of our algorithm compared to the UCB algorithm. The environment is generated as before, with a horizon of \( T = 100{,}000 \), and the experiment is repeated 10 times.


\subsection{Real-World Simulation}

The dataset used in this study is the Primary Biliary Cirrhosis (PBC) dataset from the Mayo Clinic \footnote{\url{https://www.openml.org/d/200}}, containing 418 observations and 19 variables. Collected over a 10-year span (1974–1984), it focuses on a randomized, placebo-controlled trial of D-penicillamine for treating PBC, and includes both trial participants and observational data from non-participants.

To simulate a real-world setting, we structured the data as follows: the \textbf{Z1} variable (1 for D-penicillamine, 2 for placebo) was treated as the \textit{arms of the bandit}, representing treatment groups. The \textbf{X} variable, denoting the time in days from registration to death, liver transplantation, or censoring, was used as the outcome. The \textbf{D} variable, indicating whether \textbf{X} measures time until death (1) or censoring (0), served as the \textit{mediator}.

The \textbf{D} mediator captures whether the time interval \textbf{X} is associated with death or censoring, offering key insights into the progression of the disease and the effect of treatment. This setup allows us to model the pathways from treatment to outcome, where \textbf{Z1} represents the action taken, \textbf{X} is the reward (days survived), and \textbf{D} explains the intermediate state between treatment and survival or death.

Applying the MAR algorithm to this MAR bandit environment yielded results consistent with those seen in synthetic data, as shown in Figure \ref{fig:RW-1}.

\end{document}